\documentclass{article}
\usepackage{iclr2027_conference,times}

\usepackage[utf8]{inputenc}
\usepackage[T1]{fontenc}
\usepackage{hyperref}
\usepackage{url}
\usepackage{booktabs}
\usepackage{amsfonts}
\usepackage{xcolor}
\usepackage{todonotes}
\usepackage{adjustbox}
\usepackage{caption}
\usepackage{amsmath,amssymb,amsthm,mathtools}
\usepackage{enumitem}
\usepackage{graphicx}
\usepackage{placeins}
\usepackage{thm-restate}
\usepackage{wrapfig}
\usepackage{tikz}
\usetikzlibrary{positioning,arrows.meta,calc}

\definecolor{cPurple}{HTML}{534AB7}
\definecolor{cPurpleBg}{HTML}{EEEDFE}
\definecolor{cPurpleDark}{HTML}{3C3489}
\definecolor{cTeal}{HTML}{1D9E75}
\definecolor{cTealBg}{HTML}{E1F5EE}
\definecolor{cTealDark}{HTML}{0F6E56}
\definecolor{cAmberBg}{HTML}{FAEEDA}
\definecolor{cAmberDark}{HTML}{854F0B}
\definecolor{cRedBg}{HTML}{FCEBEB}
\definecolor{cRedDark}{HTML}{A32D2D}
\definecolor{cGrayBg}{HTML}{F1EFE8}
\definecolor{cGrayDark}{HTML}{5F5E5A}
\definecolor{cGray}{HTML}{888780}
\definecolor{cSlate}{HTML}{6F86A8}
\definecolor{cSlateBg}{HTML}{E9EEF6}
\definecolor{cSlateDark}{HTML}{445A7C}
\definecolor{cAmber}{HTML}{D4901E}

\usepackage[most]{tcolorbox}
\usepackage{listings}

\definecolor{promptbg}{HTML}{F7F4EA}
\definecolor{promptframe}{HTML}{8A6F3D}

\newtcblisting{promptbox}[1][]{
    listing only,
    breakable,
    enhanced,
    colback=promptbg,
    colframe=promptframe,
    arc=2pt,
    boxrule=0.45pt,
    left=5pt,
    right=5pt,
    top=5pt,
    bottom=5pt,
    title=#1,
    fonttitle=\bfseries\small,
    coltitle=black,
    listing options={
        basicstyle=\ttfamily\scriptsize,
        breaklines=true,
        columns=fullflexible,
        keepspaces=true,
        showstringspaces=false
    }
}

\newtheorem{definition}{Definition}
\newtheorem{assumption}{Assumption}

\newtheorem{corollary}{Corollary}
\newtheorem{proposition}{Proposition}

\newcommand{\Y}{\mathcal{Y}}

\newcommand{\Sbb}{\mathcal{S}}
\newcommand{\Abb}{\mathcal{A}}

\newcommand{\Pbb}{\mathbb{P}}
\newcommand{\Ebb}{\mathbb{E}}

\newcommand{\eps}{\varepsilon}
\newcommand{\Out}{\mathrm{Out}}

\newcommand{\Valid}{\mathrm{Valid}}
\newcommand{\Sound}{\mathrm{Sound}}

\newcommand{\Term}{\mathrm{Term}}
\newcommand{\poly}{\mathrm{poly}}

\author{%
  Varun Sunkaraneni\thanks{Equal contribution. Random coin flip determined author order between the first two authors.} \\
  Texas A\&M University \\
  \And
  Pierfrancesco Beneventano$^*$ \\
  MIT \\
  \And
  Riccardo Neumarker \\
  MIT \\
  \And
  Tomaso Poggio \\
  MIT \\
  \And
  Tomer Galanti$^*$\thanks{Corresponding authors: 
  \texttt{pierb@mit.edu}, 
  \texttt{galanti@tamu.edu}. } \\
  Texas A\&M University \\
}

\iclrfinalcopy

\title{Inference-Time Amplification of Weak Reasoning Models}

\begin{document}

\maketitle

\begin{abstract}
How much of the capability of a reasoning model exposed by repeated sampling can an imperfect selector recover? We study this problem as {\em inference-time amplification}, separating {\em coverage}---whether a useful candidate is generated---from {\em identifiability}---whether it can be recognized among the alternatives. We show that better coverage alone does not guarantee better selection, while even imperfect local selection signals can be amplified through repetition. We also characterize proposal-side blind spots that additional sampling cannot overcome. On SWE-bench Verified, a single GPT-5.4 nano trajectory solves 67.0\% of tasks, while eight proposals contain a correct patch on 79.0\%; critic--comparator selection reaches 76.4\%, recovering 78.3\% of the oracle-exposed gain. The same pattern appears across proposer families and in mathematical reasoning, competitive programming, and formal proof. On MATH L4--5 and GSM-Plus, our selector recovers 85\% and 77\% of the oracle gap and substantially outperforms majority vote. Holding proposal coverage fixed, providing the selector with additional repository information also improves recovery. Conversely, when the available selection signal is weak, redundant, or misaligned, amplification can diminish or fail. These results identify selection quality as the key factor governing how much capability exposed by repeated sampling can actually be recovered.
\end{abstract}

\section{Introduction}
\vspace{-0.06cm}
A reasoning model may solve a problem on one attempt and fail on the next. Repeated sampling exposes this intermittent capability, but using it requires selecting the correct output.
The central question of this paper is therefore:
\begin{center}
    \textit{When can an imperfect selector recover the capability exposed by repeated sampling? \\ 
    Can we model that mathematically?}
\end{center}
On SWE-bench Verified~\citep{Jimenez2024SWEbench}, a single
\texttt{GPT-5.4 nano} trajectory solves \(67.0\%\) of tasks, while among eight
independent trajectories at least one correct patch appears on \(79.0\%\).
\begin{figure}[t]
\vspace{-0.15cm}
    \centering
    \includegraphics[width=0.7\linewidth]{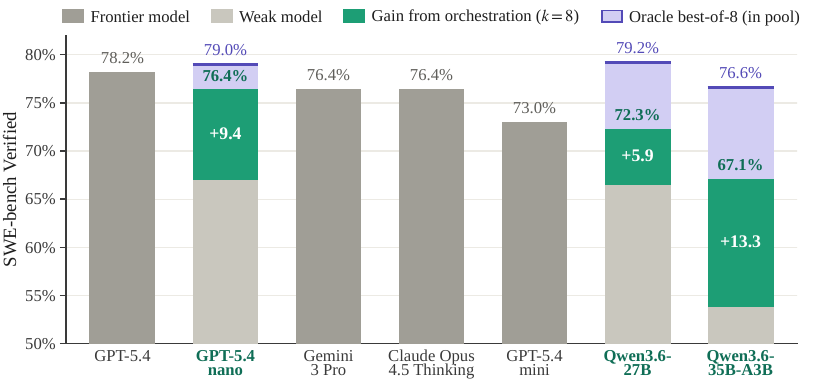}
    \caption{
    \textbf{Inference-time amplification on SWE-bench Verified.}
    With eight \texttt{GPT-5.4 nano} proposals, oracle coverage reaches
    \(79.0\%\) and our selector reaches \(76.4\%\).
    The \(67.0\%\) nano value is a separate single-trajectory reference;
    stronger models are shown for context.
    }
    \label{fig:teaser_orchestration_scaling}
\end{figure}
Our selector reaches \(76.4\%\). This gap exposes substantial
\emph{latent proposal-pool capability}: correct solutions are often generated
even when a single generation fails, but their presence is not enough. The
system must still identify the correct candidate without access to the hidden
verifier.

This is the selection problem at the heart of inference-time amplification.
Existing methods use majority vote, self-consistency, reward models, learned
verifiers, LLM-as-a-judge, or task-specific checks to choose among sampled
candidates
\citep{Wang2023,Chen2023UniversalSelfConsistency,Cobbe2021,
Lightman2024,Jiang2023LLMBlender,Huang2025BestOfN}.
Yet an oracle best-of-\(k\) result only shows that a useful candidate exists;
it says nothing about whether a deployable selector can recognize it.
Likewise, sampling alone does not solve the problem: majority voting can fail
when correct answers are diverse or infrequent, and a learned judge may itself
be weak or misaligned.

To answer our question, we separate two requirements.
\emph{Coverage} asks whether the proposal pool contains a useful candidate;
\emph{identifiability} asks whether the available evidence makes that candidate
distinguishable from the alternatives. These are logically different
resources. More samples can increase coverage without providing any information
about which sample is useful. We make this separation formal by constructing
task families with arbitrarily high proposal coverage but no uniform positive
critic or comparator edge. Thus repeated sampling creates an opportunity for
amplification, but cannot by itself realize it.

The positive result is that the selection signal need not be strong. Suppose
that, on a useful--unhelpful pair, a comparator chooses the useful candidate
with probability at least \(1/2+\sigma\), for some \(\sigma>0\). Repeated
independent comparisons amplify this weak local edge; critics can analogously
filter candidates using unary evidence. When the proposal mechanism also has
nonzero useful mass, repetition therefore amplifies both sides of the problem:
sampling increases the chance that a useful candidate appears, while repeated
selection increases the chance that it is recovered. We then show how these
local guarantees compose along adaptive reasoning trajectories when useful
actions make bounded progress toward a verified solution.

This decomposition also identifies the fundamental ceiling. If no useful
candidate is generated, no selector can recover one. We characterize such
proposal-side blind spots and show that oracle best-of-\(k\) converges to a
support-limited ceiling rather than necessarily to \(100\%\). At finite
\(k\), the oracle gap and the implemented selector together distinguish two
failure modes: \emph{proposal failure}, where no correct candidate appears, and
\emph{selection failure}, where a correct candidate appears but is not chosen.
This makes oracle-gap recovery a natural measure of how much sampled capability
a practical selector actually realizes.

{\bf Contributions.\enspace}
Our theoretical and empirical results support this picture.

\begin{enumerate}[leftmargin=2em,label=(\roman*),nosep]

    \item \textbf{We characterize the conditions and limits of inference-time
    amplification.}
    We separate proposal \emph{coverage}, local \emph{identifiability}, and
    end-to-end \emph{progress}. Coverage alone is insufficient: even arbitrarily
    high useful-proposal probability need not provide any information about which
    proposal is useful. Conversely, any uniformly positive local comparison edge
    can be amplified by repetition, and bounded progress composes these local
    guarantees into end-to-end success. We further characterize proposal-side
    blind spots that impose an irreducible ceiling on amplification.

    \item \textbf{We show that repeated sampling exposes substantial capability
    that practical selectors can recover.}
    On SWE-bench Verified, eight \texttt{GPT-5.4 nano} proposals raise the oracle
    from \(67.0\%\) pass@1 to \(79.0\%\), while critic--comparator selection
    reaches \(76.4\%\), recovering \(78.3\%\) of the exposed gain. On MATH L4--5
    and GSM-Plus, selection recovers \(85\%\) and \(77\%\) of the oracle gap and
    substantially outperforms majority vote. A controlled GSM-Plus scaling study
    further shows that pairwise comparison continues to recover roughly
    \(77\)--\(80\%\) of the gain exposed as the proposal pool grows.

    \item \textbf{We show that identifiability determines whether exposed
    capability becomes usable.}
    Across domains, stronger measured comparison edges coincide with greater
    oracle-gap recovery. More directly, holding a Qwen3.6-27B SWE-bench proposal
    pool fixed, repository context raises the comparator edge from \(0.04\) to
    \(0.09\) and recovery from \(24\%\) to \(40\%\). Conversely, direct
    verification can make learned selection unnecessary, simple heuristics can
    outperform a weak selector, and a misaligned comparator can reduce
    performance. Thus a large oracle gap creates an opportunity for amplification,
    not a guarantee.

\end{enumerate}

\section{Related Work}
\label{sec:related}

{\bf Boosting and inference-time amplification.\enspace}
Classical boosting converts weak predictive edges into strong predictors under
supervised feedback
\citep{Schapire1990,Freund1995Boosting,FreundSchapire1997}.
Related work studies LLMs as weak learners, weak-to-strong generalization, and
boosting-style model combinations
\citep{Manikandan2023,Burns2024,Chen2025LLMBoost}.
Our setting is different: model weights remain fixed, and amplification must
arise entirely at inference time from repeated proposal and imperfect local
selection.

{\bf Sampling and test-time scaling.\enspace}
Self-consistency, rationale ensembles, multi-agent voting, and repeated
sampling improve performance by generating multiple candidates
\citep{Wang2023,Wang2022RationaleEnsembles,
Chen2023UniversalSelfConsistency,Li2024MoreAgents,
Brown2024LargeLanguageMonkeys}. Recent work studies scaling with additional
model calls and best-of-\(N\) inference
\citep{Chen2024MoreCalls,Huang2025BestOfN,
Zhang2025TestTimeScalingSurvey}. We separate the capability exposed by
sampling from the ability to recover it: a correct candidate can be present in
the pool without being identifiable by the available selector.

{\bf Selection, verifiers, and judges.\enspace}
Learned verifiers, process reward models, pairwise rankers, verifier ensembles,
LLM judges, and tool-backed checking provide increasingly sophisticated ways to
select among candidate solutions
\citep{Cobbe2021,Lightman2024,Wang2024MathShepherd,
Jiang2023LLMBlender,Lifshitz2025MAV,SaadFalcon2025Weaver,
Zheng2023,Lambert2024RewardBench,Gou2024CRITIC,Dhuliawala2024CoVe}.
Rather than proposing another scoring rule, we ask when an \emph{imperfect}
selection signal is sufficient for amplification. Our analysis formalizes this
through local identifiability: a weak positive comparison edge can be amplified
by repetition, whereas proposal coverage alone cannot create such an edge.

{\bf Search, agents, and verifier-backed reasoning.\enspace}
Tree of Thoughts, RAP, LATS, ReAct, Reflexion, Self-Refine, and related
compound or multi-agent systems use search, feedback, planning, and tool
interaction over intermediate states
\citep{Yao2023ToT,Hao2023RAP,Zhou2023LATS,Yao2023ReAct,
Shinn2023Reflexion,Madaan2023SelfRefine,Li2023CAMEL,Wu2023AutoGen,
Hong2024MetaGPT,Wang2024MixtureAgents,SaadFalcon2024Archon,
Guha2024Smoothie}.
Code and formal reasoning make selection especially measurable because tests,
execution, types, or proof checkers can distinguish candidate artifacts
\citep{Li2022AlphaCode,Chen2021Codex,Jimenez2024SWEbench,
Yang2024SWEagent,Zhang2024AutoCodeRover,Xia2025Agentless}.
Related work on debate, prover--verifier games, and scalable oversight also
uses weaker local judgments to support harder decisions
\citep{Irving2018,Christiano2018,Du2024,Liang2024MAD,
BrownCohen2024,Anil2021,Kirchner2024,Kenton2024}.
We isolate a narrower mechanism: how proposal coverage, local identification
edges, and bounded progress combine to make inference-time amplification
possible.


\section{Verifier-Backed Committee Search}
\label{sec:framework}

We formalize inference-time amplification as a selection problem. Repeated
proposal calls may produce a useful next action, but the system must identify it
without access to the final hidden verifier. The theory therefore separates
\emph{proposal coverage} from \emph{local identifiability}.

We model reasoning as search over partial solutions. For an input \(x\), a
state records the current partial solution, an action transforms that state,
and a terminal state produces an output checked by a final verifier.

\begin{definition}[Valid-state system]
\label{def:vss}
For an input \(x\), let \(\Sbb_x\) be a countable state space with initial
state \(s_0(x)\), valid states \(\Valid_x\subseteq\Sbb_x\), terminal states
\(\Term_x\subseteq\Valid_x\), action sets \(\Abb_x(s)\), and transition map
\(\Phi_x(s,a)\). Let \(\Out_x:\Term_x\to\Y_x\) map terminal states to outputs,
and let \(R_x:\Y_x\to\{0,1\}\) be the terminal verifier.

A rank function \(d_x:\Valid_x\to\{0,1,\ldots,L_x\}\) satisfies
\(d_x(s)=0\) on terminal states and \(d_x(s)\geq1\) on valid nonterminal
states. We require \(s_0(x)\in\Valid_x\) and
\(R_x(\Out_x(s))=1\) for every \(s\in\Term_x\).
\end{definition}

A state is \emph{reachable} if it can be obtained from \(s_0(x)\) by a finite
sequence of actions. At a reachable valid nonterminal state \(s\), call an
action \emph{progressing-sound} if it preserves validity and strictly decreases
the rank:
\[
\Sound_x(s)
:=
\left\{
a\in\Abb_x(s):
\Phi_x(s,a)\in\Valid_x,\;
d_x(\Phi_x(s,a))<d_x(s)
\right\}.
\]
Thus any trajectory of progressing-sound actions reaches a terminal state
within at most \(L_x\) steps.

This notion separates local usefulness from final verification. In SWE-bench,
for example, a progressing-sound edit need only preserve the existence of a
hidden-test-passing completion while making progress toward it. Visible tests,
type checks, linters, or learned judgments may provide local evidence without
certifying final correctness.

{\bf Committee protocol.\enspace}
At each reachable valid nonterminal state, \(\Pi_{k,m,r}\) draws \(k\)
candidate actions, applies \(m\) independent critic calls to each, and removes
any candidate rejected by a critic. If none survives, the protocol fails.
Otherwise, each surviving pair receives \(r\) comparator votes; majority vote
orients each pair, and a Copeland winner is selected and applied. The process
repeats until a terminal state is reached.

At termination the protocol returns the corresponding output under \(\Out_x\).
We write \(\bot\) for the no-survivor failure and define \(R_x(\bot)=0\).


Local success requires two conditions: a progressing-sound action must appear
in the pool, and the selection signal must distinguish it from unsound
alternatives. We formalize these as coverage and identifiability.

\begin{assumption}[Per-state local coverage]
\label{ass:H1}
For each problem size \(N\), there is a proposer portfolio \(P_N\) of
polynomial size such that, for every reachable valid nonterminal state \(s\)
of every size-\(N\) input, at least one proposer \(p_s\in P_N\) generates a
progressing-sound action with probability at least some constant
\(\alpha_0>0\).
\end{assumption}

Assumption~\ref{ass:H1} guarantees that useful actions have nontrivial
probability of appearing, but says nothing about which proposal is useful.
Sampling can therefore amplify coverage without solving selection.

\begin{assumption}[Efficient local identifiability]
\label{ass:H2}
At every reachable valid nonterminal state \(s\), there are polynomial-time
randomized procedures \(\mathsf{Crit}_s\) and \(\mathsf{Comp}_s\), and
constants \(\beta_0,\sigma_0>0\), such that:
\begin{enumerate}[label=(\roman*),nosep,leftmargin=2em]
\item if \(a\in\Sound_x(s)\), then \(\mathsf{Crit}_s(a)\) never rejects;
\item if \(a\notin\Sound_x(s)\), then
\(\Pbb(\mathsf{Crit}_s(a)=\textsc{reject})\geq\beta_0\); and
\item if \(a\in\Sound_x(s)\) and \(b\notin\Sound_x(s)\), then
\(\Pbb(\mathsf{Comp}_s(a,b)=a)\geq\tfrac12+\sigma_0\).
\end{enumerate}
\end{assumption}

Assumption~\ref{ass:H2} supplies the information needed for selection. In
particular, the comparator need only have a positive edge \(\sigma_0\) on
progressing-sound--unsound pairs; repetition will amplify this weak preference.

Assumption~\ref{ass:H1} concerns generation; Assumption~\ref{ass:H2} concerns
selection. Neither is enough on its own. The one-sided critic is a clean
sufficient condition rather than a claim about practical learned critics;
Appendix~\ref{app:amplification} allows false rejections. We next show why
additional sampling cannot create an identification signal, and then show how
any positive signal can be amplified.

\section{Coverage and Identifiability}
\label{sec:bridge}

More samples can make a useful action increasingly likely to appear without
making it any easier to recognize. Coverage and identifiability are therefore
separate resources: the former is supplied by the proposer, while the latter
requires information about candidate quality beyond their mere occurrence.

\begin{restatable}[Coverage does not imply local identifiability]
{proposition}{blackboxprop}
\label{prop:blackbox}
For every \(M\geq2\), there is a one-step task family with action set
\(\Abb=\{1,\ldots,M\}\) and hidden parameter
\(\theta\in\{1,\ldots,M\}\) such that the proposer is
\(\mathrm{Unif}(\Abb)\) in every world and
\(\Sound_\theta=\Abb\setminus\{\theta\}\). Even with arbitrarily many samples
from this proposer, a selector that receives no information about \(\theta\)
cannot have a uniform positive critic or comparator edge over all worlds.
\end{restatable}

The construction isolates the obstruction. A proposal is sound with
probability \(1-1/M\), so useful actions are abundant, yet the proposal
distribution is identical in every hidden world. Additional samples therefore
increase neither the information available to a critic nor the edge of a
comparator. Reliable selection requires an additional signal---for example
execution, tests, proof checking, type information, constraints, or a learned
reviewer with task-relevant information.

Sampling can nevertheless amplify the part of the problem it controls:
coverage.

\begin{restatable}[Coverage amplification from a proposer portfolio]
{theorem}{bridgethm}
\label{thm:bridge}
Fix a reachable valid nonterminal state \(s\) and suppose
Assumption~\ref{ass:H1} holds. Query the proposers in \(P_N\) in round-robin
order with fresh independent randomness. After \(k\) calls, the probability
that the candidate pool contains a progressing-sound action is at least
\(1-(1-\alpha_0)^{\lfloor k/|P_N|\rfloor}\). In particular, this probability
is at least \(1-\delta_{\mathrm{prop}}\) whenever
\(k\geq |P_N|\lceil
\alpha_0^{-1}\log(1/\delta_{\mathrm{prop}})
\rceil\).
\end{restatable}

Theorem~\ref{thm:bridge} says precisely what repeated sampling can guarantee:
a useful candidate becomes increasingly likely to appear. It says nothing
about whether the system can recognize that candidate once it is present.
That is the separate role of identifiability.

{\bf Local verification as a source of identifiability.\enspace} The selection signal need not come from a stronger learned model. A local verifier can provide it directly. The important distinction is that a verifier of completed outputs is not automatically useful for choosing among
intermediate candidates; it must provide information about the candidates
being compared.

\begin{assumption}[One-sided local verifier]
\label{ass:D3}
At every reachable valid nonterminal state \(s\), there is a polynomial-time
randomized verifier \(V_x(s,a)\in\{\textsc{accept},\textsc{reject}\}\) that
accepts every \(a\in\Sound_x(s)\) almost surely and rejects every
\(a\notin\Sound_x(s)\) with probability at least \(1-\nu\).
\end{assumption}

Such a verifier immediately supplies both forms of local evidence. It can be
used directly as a critic, and it induces a comparator by verifying both
candidates, choosing the unrejected candidate when exactly one is rejected,
and otherwise breaking the tie uniformly.

\begin{restatable}[Local verification implies identifiability]
{corollary}{verifierbridgecor}
\label{cor:bridge_D3}
Under Assumption~\ref{ass:D3}, Assumption~\ref{ass:H2} holds with
\(\beta_0=1-\nu\) and \(\sigma_0=(1-\nu)/2\).
\end{restatable}

The verifier can be used directly as a critic. It also induces a comparator:
verify both candidates, choose the unrejected one when exactly one is rejected,
and otherwise break the tie uniformly. On a sound--unsound pair this gives
comparison success at least \(1/2+(1-\nu)/2\). Thus amplification requires an
informative signal, not necessarily a powerful learned selector. A weak
positive signal is enough for repetition to help.

Appendix~\ref{app:bridge} gives the formal proof and shows why a verifier that
checks only completed outputs need not provide the same local information.

\section{Local-to-Global Amplification}
\label{sec:amplification}

We now give the positive amplification result. A local decision can fail for
two reasons: no progressing-sound action is proposed, or such an action is
proposed but the selector fails to recover it. The first is a coverage error;
the second is an identification error. Repetition suppresses these two errors
through different mechanisms.

Fix a reachable valid nonterminal state \(s\). Let
\(\eps_{\mathrm{prop}}(k;s)\) denote the probability that none of the \(k\)
proposals lies in \(\Sound_x(s)\), and let \(\eps_{\mathrm{loc}}(s)\) denote
the probability of local failure, conditional on the protocol being at \(s\):
either no candidate survives or the selected action lies outside
\(\Sound_x(s)\).

\begin{restatable}[Local error decomposition]{theorem}{localerrorthm}
\label{thm:local}
Suppose the local-identifiability conditions hold at \(s\) with critic edge
\(\beta\) and comparator edge \(\sigma\), and let \(r\) be odd. Conditional on
the current state and the proposed candidates, assume that repeated critic
calls are independent, that repeated comparator votes for each pair are
independent, and that the comparator randomness is fresh relative to the
critic randomness. Then
\[
\eps_{\mathrm{loc}}(s)
\leq
\eps_{\mathrm{prop}}(k;s)
+
k^2(1-\beta)^m e^{-2r\sigma^2}
\leq
\eps_{\mathrm{prop}}(k;s)
+
k^2e^{-\beta m-2r\sigma^2}.
\]
If, in addition, the \(k\) proposal calls are conditionally independent and
each is progressing-sound with probability at least \(\alpha\), then
\(\eps_{\mathrm{prop}}(k;s)\leq(1-\alpha)^k\leq e^{-\alpha k}\).
\end{restatable}

The bound makes the amplification mechanism explicit. Increasing \(k\) reduces
the probability that no useful action appears. Once one is present, critic
repetition suppresses unsound survivors as \(e^{-\beta m}\), while repeated
comparison amplifies even a weak positive edge \(\sigma>0\), with error
decaying as \(e^{-2r\sigma^2}\). Thus the selector need not be highly accurate
in a single comparison: a uniformly better-than-chance local preference can be
amplified by repetition. The factor \(k^2\) comes from a conservative union
bound over sound--unsound candidate pairs.

This establishes reliable selection at one state. To obtain an end-to-end
guarantee, these local decisions must compose. If every selected action is
progressing-sound, validity is preserved and the integer-valued rank decreases
strictly, so the protocol reaches a terminal state within at most \(L_x\)
steps.

Define the end-to-end failure probability
\(\mathrm{err}_x(k,m,r):=
\Pbb(R_x(\Pi_{k,m,r}(x))=0)\).

\begin{restatable}[Adaptive cumulative-error bound]{lemma}{cumulativelem}
\label{lem:cumulative}
Let \(\mathcal F_t\) denote the history available before the decision at step
\(t\). Suppose that, whenever the protocol is at a reachable valid nonterminal
state, the conditional probability of local failure at step \(t\) is at most
\(\eps_t\). Then
\(\mathrm{err}_x(k,m,r)\leq
\Ebb[\sum_{t=0}^{L_x-1}\eps_t]\).
In particular, if \(\eps_t\leq\eps\) uniformly along the trajectory, then
\(\mathrm{err}_x(k,m,r)\leq L_x\eps\).
\end{restatable}

No independence across reasoning steps is required: the state, proposal
distribution, and local error rate may depend arbitrarily on the preceding
trajectory. It is enough that the conditional failure probability is
controlled at each state actually reached.

\begin{corollary}[Global amplification bound]
\label{cor:global_amplification}
Suppose
\(\eps_{\mathrm{prop}}(k;s)\leq\eps_{\mathrm{prop}}(k)\) uniformly over
reachable valid nonterminal states, and suppose the critic and comparator edges
are uniformly bounded below by \(\beta\) and \(\sigma\). Then
\[
\mathrm{err}_x(k,m,r)
\leq
L_x\!\left(
\eps_{\mathrm{prop}}(k)
+
k^2e^{-\beta m-2r\sigma^2}
\right).
\]
Consequently, the protocol succeeds with probability at least \(1-\delta\)
whenever the local proposal-plus-selection error in parentheses is at most
\(\delta/L_x\).
\end{corollary}

The corollary formalizes the central claim: inference-time amplification
requires both useful proposals and an informative selector. Proposal width
\(k\) controls whether a useful action appears; critic and comparator budgets
\(m,r\) control whether it is recovered; and trajectory depth \(L_x\)
determines how local errors accumulate. Explicit polynomial-resource choices
appear in Appendix~\ref{app:poly_resource}.

\section{Blind Spots and Oracle Ceilings}
\label{sec:blindspots}

Selection can recover only capability that the proposer exposes. Even a perfect
selector cannot help when the proposal mechanism assigns zero probability to a
useful action. We therefore ask how far repeated sampling alone can reduce the
coverage term \(\eps_{\mathrm{prop}}(k;s)\).

\begin{restatable}[Oracle miss and blind-spot floor]{lemma}{robustlem}
\label{lem:robust}
Fix a reachable valid nonterminal state \(s\). Let \(Z\) capture latent
heterogeneity shared across proposal calls. Conditional on \(Z\), suppose the
\(k\) calls are independent and each proposes a progressing-sound action with
probability \(q_s(Z)\). Then
\[
\eps_{\mathrm{prop}}(k;s)
=
\Ebb[(1-q_s(Z))^k]
=
B_s+R_k(s),
\]
where \(B_s:=\Pbb(q_s(Z)=0)\) and
\(R_k(s):=\Ebb[(1-q_s(Z))^k\mathbf 1\{q_s(Z)>0\}]\).
Moreover, \(R_k(s)\to0\), so
\(\eps_{\mathrm{prop}}(k;s)\to B_s\).
\end{restatable}

Here \(R_k(s)\) is finite-sample error: a useful action has positive probability
but happens not to appear in the \(k\) draws. It vanishes as \(k\) grows. By
contrast, \(B_s\) represents cases on which the proposer has zero probability
of producing a useful action. More samples from the same proposer cannot fix
those cases. Consequently, finite-\(k\) oracle failure upper-bounds rather than
identifies the asymptotic blind-spot floor. Heterogeneous portfolios and
finite-\(k\) rates appear in Appendix~\ref{app:robustness}.

Let \(B:=\sup_s B_s\) and \(R_k:=\sup_s R_k(s)\), where the supremum ranges
over reachable valid nonterminal states. Combining
Lemma~\ref{lem:robust} with Corollary~\ref{cor:global_amplification} gives
\[
\mathrm{err}_x(k,m,r)
\leq
L_x\Bigg[
\underbrace{B}_{\textnormal{blind-spot floor}}
+
\underbrace{R_k}_{\textnormal{finite-width proposal error}}
+
\underbrace{k^2e^{-\beta m-2r\sigma^2}}_{\textnormal{selection error}}
\Bigg].
\]
This yields the full decomposition. More proposal samples reduce finite-width
coverage error; more critic or comparator effort reduces selection error; and
neither can remove a support-level blind spot. Once selection error is small,
further improvement must come from proposal coverage rather than additional
judging.

{\bf Terminal oracle curves.\enspace}
The same decomposition has a direct task-level analogue in our experiments.
Let \(X\sim\mathcal D\) and
\(q_P(x):=\Pbb_{Y\sim P(\cdot\mid x)}(R_x(Y)=1)\). Then pass@1 is
\(p_1(P)=\Ebb[q_P(X)]\), while oracle best-of-\(k\) is
\(p_{\mathrm{oracle}}(k;P)=\Ebb[1-(1-q_P(X))^k]\). As \(k\to\infty\), the
oracle converges to \(1-\Pbb(q_P(X)=0)\).

Pass@1 therefore measures one-shot proposal success, whereas oracle
best-of-\(k\) measures how often a correct output is present in a
\(k\)-candidate pool. Their gap is the capability exposed by additional
sampling under perfect selection. At finite \(k\), however, oracle failure
still combines true blind spots with ordinary finite-sampling misses; even an
observed plateau is evidence of saturation rather than a direct estimate of
the asymptotic blind-spot mass.

{\bf Oracle-gap recovery.\enspace}
For a selector evaluated on the same \(k\)-candidate pool, let
\(p_{\mathrm{system}}(k,m,r;P)\) denote its success probability. Whenever
\(p_{\mathrm{oracle}}(k;P)>p_1(P)\), define
\[
\mathrm{Rec}(k,m,r;P)
:=
\frac{
p_{\mathrm{system}}(k,m,r;P)-p_1(P)
}{
p_{\mathrm{oracle}}(k;P)-p_1(P)
}.
\]
The denominator measures the capability exposed by sampling under perfect
selection; the numerator measures how much of that gain the implemented
selector realizes. Oracle-gap recovery therefore directly connects the theory
to the experiments: a large oracle gap with low recovery identifies a selection
bottleneck, whereas high recovery shifts the remaining bottleneck toward
proposal coverage.

\section{Experiments}
\label{sec:experiments}

Our experiments measure separately whether repeated sampling produces a correct
candidate and whether the selector can recover it. Oracle best-of-\(k\) is used
only as a hindsight measure of proposal-pool coverage and is never available to
the selector.

\subsection{Setup}
\label{subsec:experiments_setup}

SWE-bench Verified contains \(500\) repository-level software-engineering tasks
scored by held-out tests. In the headline experiment, we generate \(k=8\)
independent \texttt{GPT-5.4 nano} Mini-SWE-agent trajectories per task and
select among their patches using binary critics followed by pairwise
comparators. Hidden-test outcomes are used only for final evaluation and oracle
diagnostics; they are never exposed to the selector.

A single headline trajectory resolves \(67.0\%\) of tasks. Full
critic--comparator orchestration reaches \(76.4\%\), while an oracle that
chooses a passing patch whenever one is present reaches \(79.0\%\). The
implemented system therefore recovers \(78.3\%\) of the improvement exposed by
eight proposals. Among the tasks the system still misses, \(89.0\%\) contain no
passing patch in the eight-proposal pool and \(11.0\%\) are oracle-reachable
selection misses. Thus, at this budget, the dominant residual is proposal
coverage rather than selection.

Additional analyses reuse fixed archived proposal and judgment pools; matched
baseline definitions, cache provenance, prompts, and implementation details appear in Appendix~\ref{app:swebench_implementation}.

\subsection{Sampling Creates Opportunity; Selection Determines Recovery}
\label{subsec:experiments_results}

\begin{figure}[t]
\centering

\includegraphics[width=0.72\linewidth]
{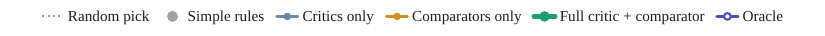}
\vspace{-0.35em}

\begin{minipage}[t]{0.32\linewidth}
  \centering
  \includegraphics[width=\linewidth]
  {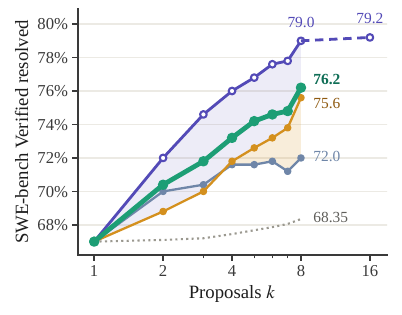}\\[-0.25em]
  {\small (a) SWE-bench scaling}
\end{minipage}\hfill
\begin{minipage}[t]{0.32\linewidth}
  \centering
  \includegraphics[width=\linewidth]
  {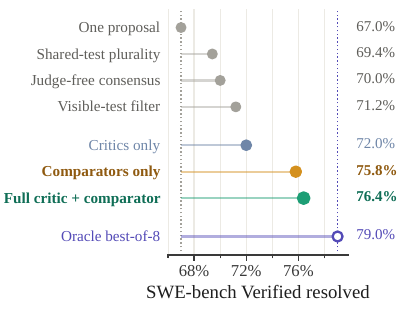}\\[-0.25em]
  {\small (b) SWE-bench at \(k=8\)}
\end{minipage}\hfill
\begin{minipage}[t]{0.32\linewidth}
  \centering
  \includegraphics[width=\linewidth]
  {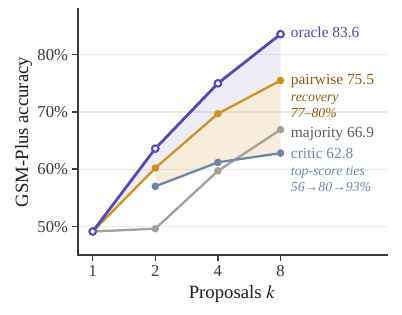}\\[-0.25em]
  {\small (c) GSM-Plus scaling}
\end{minipage}

\vspace{-0.35em}
\caption{\textbf{Sampling exposes capability; selection determines how much is recovered.}
(a) On SWE-bench Verified, increasing proposal width expands oracle coverage,
but selectors differ substantially in how well they exploit the larger pool.
(b) At \(k=8\), simple selection rules leave substantial oracle-exposed
capability unused, while pairwise comparison closes most of the remaining gap.
(c) On GSM-Plus, pairwise selection continues to recover a nearly constant
fraction of the capability exposed as the proposal pool grows, while pointwise
criticism increasingly loses resolution.}
\label{fig:component_scaling}
\end{figure}

Figure~\ref{fig:component_scaling} shows the coverage--selection decomposition
directly. On SWE-bench, increasing \(k\) expands oracle coverage, but selectors
exploit that opportunity differently: pointwise criticism saturates earlier,
while pairwise comparison continues to improve. Simple rules also leave
substantial oracle headroom at \(k=8\). Doubling the proposal budget from
\(k=8\) to \(16\) raises the oracle only from \(79.0\%\) to \(79.2\%\),
showing sharply diminishing finite-width returns without identifying the
asymptotic blind-spot floor.

On GSM-Plus, pairwise comparison recovers roughly \(77\)--\(80\%\) of the
oracle-exposed gain across \(k=2,4,8\), whereas critic-only recovery falls as
the top-score tie rate rises from \(56\%\) to \(93\%\). This loss of score
resolution explains why pointwise criticism benefits less from larger proposal
pools. Full subset-averaged results and bootstrap intervals appear in
Appendix~\ref{app:additional_empirical}.

\subsection{Identifiability Governs Recovery}
\label{subsec:identifiability_empirical}

For empirical comparator diagnostics, let \(\tau\) be the observed tie rate on
correct--incorrect candidate pairs and let \(a\) be accuracy conditional on a
decisive judgment. We report
\(\hat{\sigma}:=(1-\tau)(a-\tfrac12)\); equivalently, assigning ties half
credit gives overall pairwise accuracy \(\tfrac12+\hat{\sigma}\), matching the
comparator edge in Assumption~\ref{ass:H2}.

\begin{figure}[t]
\centering

\newlength{\panelheight}
\setlength{\panelheight}{0.205\textheight}

\begin{minipage}[t]{0.38\linewidth}
  \centering
  \begin{minipage}[c][\panelheight][c]{\linewidth}
    \centering
    \includegraphics[width=\linewidth]
    {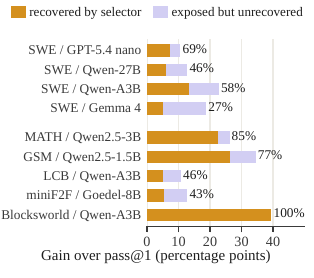}
  \end{minipage}
  {\small (a) Exposed and recovered gain}
\end{minipage}
\hfill
%
\begin{minipage}[t]{0.29\linewidth}
  \centering
  \begin{minipage}[c][\panelheight][c]{\linewidth}
    \centering
    \includegraphics[width=\linewidth]
    {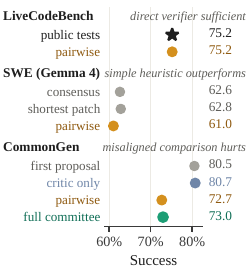}
  \end{minipage}
  {\small (b) Boundary cases}
\end{minipage}
\hfill
%
\begin{minipage}[t]{0.30\linewidth}
  \centering
  \begin{minipage}[c][\panelheight][c]{\linewidth}
    \centering

    {\scriptsize
    \renewcommand{\arraystretch}{1.08}
    \setlength{\tabcolsep}{1.8pt}

    \begin{adjustbox}{max width=\linewidth}
    \begin{tabular}{@{}lccc@{}}
      \toprule
       & Patch only & + Repo. & $\Delta$ \\
      \midrule
      Comparator edge $\hat{\sigma}$
        & 0.04 & 0.09 & +0.05 \\
      Solve rate
        & 67.0\% & 69.6\% & +2.6 pp \\
      Oracle-gap recovery
        & 24\% & 40\% & +16 pp \\
      \midrule
      \multicolumn{4}{c}{\textit{Proposal pool and oracle held fixed}} \\
      \multicolumn{4}{c}{Qwen3.6-27B; oracle \(=79.2\%\)} \\
      \bottomrule
    \end{tabular}
    \end{adjustbox}
    }

  \end{minipage}\\[-0.15em]
  {\small (c) Fixed-coverage intervention}
\end{minipage}

\caption{\textbf{Identifiability governs recovery of sampled capability.}
(a) Across proposer--task configurations, repeated sampling can expose
substantial additional capability, but selectors recover widely varying
fractions of that opportunity.
(b) Boundary cases show that learned selection can be redundant, inferior to a
simple heuristic, or harmful when its signal is misaligned with correctness.
(c) Holding both the SWE-bench proposal pool and oracle coverage fixed,
providing repository context improves the comparator edge, end-to-end solve
rate, and oracle-gap recovery.}
\label{fig:identifiability}
\end{figure}

A large oracle gap is an opportunity for amplification, not a guarantee.
Figure~\ref{fig:identifiability}a shows that the fraction of sampled capability
recovered by selection varies widely across proposer--task configurations:
MATH and GSM recover most of the exposed gain, several SWE-bench pools recover
far less, and task-aligned verification on Blocksworld recovers essentially all
of it. Full cross-model curves and controls appear in
Appendix~\ref{app:additional_empirical}.

Figure~\ref{fig:identifiability}b shows three ways this mechanism can break
down. Direct verification can make learned judging unnecessary; a simple
heuristic can outperform a weak judge; and a misaligned comparator can make
selection worse than doing nothing.

Figure~\ref{fig:identifiability}c gives a cleaner test of the role of
identifiability. We hold the Qwen3.6-27B proposal pool and its \(79.2\%\)
oracle fixed and change only the information available to the comparator.
Repository access raises \(\hat\sigma\) from \(0.04\) to \(0.09\), solve rate
from \(67.0\%\) to \(69.6\%\), and oracle-gap recovery from \(24\%\) to
\(40\%\). Better information therefore improves selection without changing
what was generated.

The effect is not specific to software repair. On MATH L4--5, selection reaches
\(71.2\%\) from \(48.8\%\) pass@1 against a \(75.0\%\) oracle. On
miniF2F-hard it reaches \(25.0\%\) from \(19.5\%\), toward a \(32.3\%\)
oracle. The latter result is suggestive rather than conclusive
(\(p=0.064\)).

\subsection{Resource and Deployment Tradeoffs}
\label{subsec:cost_accuracy}

Amplification uses resources in different ways. More proposals increase the
chance that a correct solution appears, while judging and additional task
context help identify it. Orchestration can also substitute additional
inference-time computation for a larger resident model.

\begin{figure}[t]
\centering

\begin{minipage}[t]{0.40\linewidth}
    \centering
    \includegraphics[width=\linewidth]
    {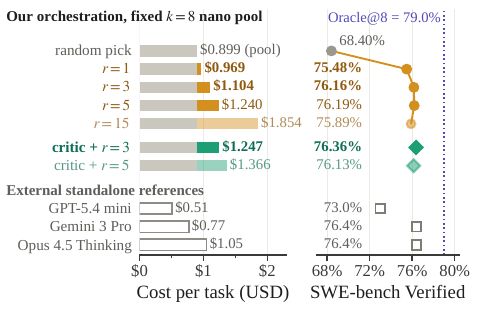}\\[-0.25em]
    {\small (a) Cost--accuracy tradeoff}
\end{minipage}
\hfill
\begin{minipage}[t]{0.28\linewidth}
    \centering
    \includegraphics[width=\linewidth]
    {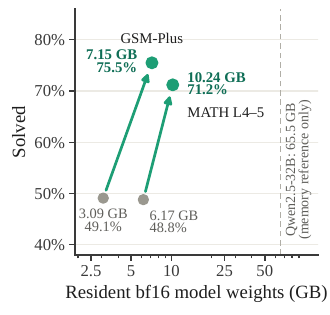}\\[-0.25em]
    {\small (b) Resident model footprint}
\end{minipage}
\hfill
\begin{minipage}[t]{0.28\linewidth}
    \centering
    \includegraphics[width=\linewidth]
    {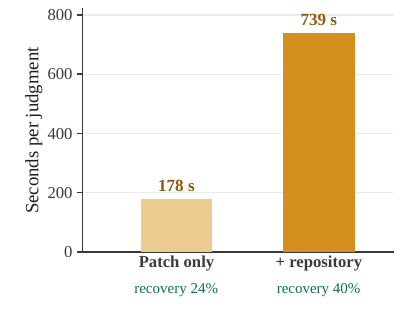}\\[-0.25em]
    {\small (c) Cost of richer evidence}
\end{minipage}

\vspace{-0.35em}
\caption{\textbf{Capability--resource tradeoffs.}
(a) Judge cost shows diminishing returns on a fixed SWE-bench pool; external
systems are contextual, unmatched references.
(b) Small proposer--judge combinations improve accuracy on MATH and GSM-Plus
with smaller resident model footprints.
(c) Repository access improves selection at higher judgment latency.}
\label{fig:resource_tradeoffs}
\end{figure}

On the fixed SWE-bench proposal pool, most of the selection gain appears at
low marginal cost. One comparator vote per pair reaches \(75.48\%\) at
\(\$0.97\) per task, and three votes reach \(76.16\%\) at \(\$1.10\);
additional judging brings little further improvement. For context, the external
Claude Opus 4.5 Thinking reference reports \(76.4\%\) at \(\$1.05\).
This is not a controlled cost comparison---provider settings, caching, and
inference budgets differ---but it shows that the amplified nano system operates
in a similar reported cost range rather than a qualitatively different one.

The other panels show two different deployment tradeoffs. On MATH and GSM, a
small proposer--judge combination produces large gains while using much less
resident model memory than the 32B reference. Repository access improves
selection on SWE-bench, but also makes each judgment substantially slower.

Wall-clock cost need not scale with logical call count. Proposals and many
judgments are independent and can run in parallel; with open-weight models they
can also be batched across candidates or tasks. We did not have the same
batching control for the primary \texttt{GPT-5.4 nano} experiments, so the
reported nano costs do not assume hypothetical batching gains. Full token,
call, latency, and cost accounting appears in
Appendix~\ref{app:additional_empirical}.

\section{Discussion and Limitations}
\label{sec:discussion}

One-shot performance can understate what is present in a model's proposal
pool: repeated sampling may expose correct solutions that selection must
recover. Our theory formalizes this distinction and its multi-step composition,
while the experiments primarily measure terminal selection.

The remaining failures point to different remedies. If a correct candidate is
present but not selected, better judging can help. If no correct candidate is
generated, better judging cannot help; the proposal mechanism itself must
change. In the headline SWE-bench experiment, most remaining failures are of
this second kind. Our empirical results mostly test terminal selection rather
than multi-step composition, finite-\(k\) oracle failure only upper-bounds the
asymptotic blind-spot floor, and practical deployment must also account for
compute, latency, and system complexity.


\newpage

\bibliographystyle{iclr2027_conference}
\bibliography{references}

\newpage

\appendix

\section{Proofs and Theory Extensions}
\label{app:theory}

This section proves the results from
Sections~\ref{sec:bridge}--\ref{sec:blindspots} and develops several extensions
that clarify the scope of the theory. Throughout, probabilities are taken over
the fresh randomness of the corresponding proposer, critic, or comparator
calls, conditional on any state, candidate pool, or history specified in the
statement.

\subsection{Coverage and Identifiability}
\label{app:bridge}

For Proposition~\ref{prop:blackbox}, the identification procedure may know the
task-family construction, the action set, and the proposer distribution. It may
also observe the proposed actions and arbitrarily many additional samples from
the same proposer. What it does not observe is the hidden world parameter
\(\theta\). The critic and comparator requirements must hold uniformly over
all values of \(\theta\).

\subsubsection{Proof of Proposition~\ref{prop:blackbox}}

\blackboxprop*

\begin{proof}
Fix \(M\geq2\). In world \(\theta\in\{1,\ldots,M\}\), let
\(\Sound_\theta=\Abb\setminus\{\theta\}\), and let the proposer distribution be
\(\mu=\mathrm{Unif}(\Abb)\), independently of \(\theta\). Thus a proposal is
sound with probability \((M-1)/M\) in every world, but the entire observable
proposal transcript has exactly the same distribution for every value of
\(\theta\).

We first rule out a one-sided critic with positive uniform rejection edge.
Fix an action \(i\). When \(\theta=i\), the action \(i\) is unsound; when
\(\theta\neq i\), it is sound. Since the procedure does not observe
\(\theta\), its observable input has the same distribution in all of these
worlds, and hence its probability of rejecting \(i\) is the same in every
world. The one-sided condition requires this rejection probability to be zero
whenever \(\theta\neq i\). It must therefore also be zero when \(\theta=i\).
Thus no one-sided critic can reject every unsound action with any uniform
positive probability.

Now consider a comparator and fix two distinct actions \(i\) and \(j\). In
world \(\theta=i\), \(j\) is sound and \(i\) is unsound, so a comparator with
edge \(\sigma>0\) must choose \(j\) with probability at least
\(1/2+\sigma\). In world \(\theta=j\), the roles reverse, and the same
procedure must choose \(i\) with probability at least \(1/2+\sigma\).
But the observable transcript, including any additional proposer samples, has
the same distribution in the two worlds. These are therefore requirements on
one and the same output distribution. Since the probabilities of choosing
\(i\) and \(j\) sum to at most one, they cannot both exceed \(1/2\).
Hence no uniform positive comparator edge exists.
\end{proof}

The construction deliberately separates \emph{coverage} from
\emph{identifiability}: useful actions are abundant, but the proposer samples
contain no information about which action is the exceptional bad one.


\subsubsection{Proof of Theorem~\ref{thm:bridge}}

\bridgethm*

\begin{proof}
Fix a reachable valid nonterminal state \(s\), and let \(p_s\in P_N\) be the
portfolio member guaranteed by Assumption~\ref{ass:H1}. Under round-robin
assignment, every portfolio member is queried at least
\(q:=\lfloor k/|P_N|\rfloor\) times. In particular, \(p_s\) is queried at
least \(q\) times.

Each such call produces a progressing-sound action with probability at least
\(\alpha_0\). Because these calls use fresh independent randomness, the
probability that all \(q\) calls to \(p_s\) fail is at most
\((1-\alpha_0)^q\). The event that the entire \(k\)-proposal pool contains no
progressing-sound action is a subset of this event, so
\[
\Pbb(\textnormal{no progressing-sound proposal})
\leq
(1-\alpha_0)^{\lfloor k/|P_N|\rfloor}.
\]
This proves the first claim.

For the second claim, suppose
\(k\geq |P_N|
\lceil\alpha_0^{-1}\log(1/\delta_{\mathrm{prop}})\rceil\).
Then
\(\lfloor k/|P_N|\rfloor
\geq\alpha_0^{-1}\log(1/\delta_{\mathrm{prop}})\).
Using \(1-u\leq e^{-u}\),
the miss probability is at most
\(e^{-\alpha_0 q}\leq\delta_{\mathrm{prop}}\), as required.
\end{proof}


\subsubsection{Proof of Corollary~\ref{cor:bridge_D3}}

\verifierbridgecor*

\begin{proof}
Use \(V_x(s,\cdot)\) directly as the critic. A progressing-sound action is
accepted almost surely, while an unsound action is rejected with probability
at least \(1-\nu\). Thus the critic satisfies Assumption~\ref{ass:H2} with
\(\beta_0=1-\nu\).

For the comparator, apply the verifier to both candidates. If exactly one is
rejected, choose the other; otherwise choose uniformly between them. Consider
a pair consisting of a progressing-sound action \(a\) and an unsound action
\(b\). The verifier never rejects \(a\). Let \(p\) be the probability that it
rejects \(b\); by assumption \(p\geq1-\nu\). The induced comparator therefore
selects \(a\) with probability
\(p+(1-p)/2=1/2+p/2\), which is at least
\(1/2+(1-\nu)/2\). Hence
\(\sigma_0=(1-\nu)/2\) is a valid comparator edge.
\end{proof}

Notice that independence between the two verifier calls is not needed for this
argument: the sound candidate is accepted almost surely, so only the rejection
probability of the unsound candidate enters the bound.


\subsubsection{Final Verification Does Not Imply Local Identification}

The existence of an efficient verifier for completed outputs is substantially
weaker than Assumption~\ref{ass:H2}. The following construction makes the
distinction explicit.

\begin{proposition}[A final verifier need not yield an efficient local critic]
\label{prop:local_verifier_barrier}
There is a polynomial-time verifiable search relation with a natural
partial-state decomposition for which deciding whether a proposed action is
progressing-sound is NP-complete, even at a reachable valid state.
Consequently, if there were a \emph{uniform} polynomial-time one-sided critic
for this decomposition whose rejection edge satisfied
\(\beta_0\geq1/\poly(N)\), then
\(\mathrm{coNP}\subseteq\mathrm{RP}\), equivalently
\(\mathrm{NP}\subseteq\mathrm{coRP}\).
\end{proposition}

\begin{proof}
Consider circuit satisfiability. An instance is a Boolean circuit \(C\), a
complete output is an assignment \(y\), and the final verifier accepts exactly
when \(C(y)=1\). Complete outputs are therefore verifiable in polynomial time.

Use partial assignments as states. Declare a partial assignment valid exactly
when it admits at least one satisfying completion. An action assigns a value to
one previously unassigned variable, and let the rank be the number of
unassigned variables. Since every action decreases the rank by one, such an
action is progressing-sound exactly when the extended partial assignment
remains valid, i.e., has a satisfying completion.

We now reduce circuit satisfiability to this local decision problem. Given an
arbitrary circuit \(\psi(u_1,\ldots,u_m)\), construct
\(C(x_0,u)=x_0\lor\psi(u)\). The empty partial assignment is valid, since
setting \(x_0=1\) satisfies \(C\). Consider at this reachable valid state the
action \(x_0\leftarrow0\). After this action, a satisfying completion exists
if and only if \(\psi\) is satisfiable. Hence this action is progressing-sound
if and only if \(\psi\) is satisfiable. Local progressing-soundness is
therefore NP-hard. It is also in NP, because a satisfying completion is a
polynomial-size certificate. Thus local progressing-soundness is NP-complete,
and local unsoundness is coNP-complete.

Now suppose there is a uniform polynomial-time one-sided critic whose rejection
probability on every unsound action is at least
\(\beta_0\geq1/\poly(N)\). Given a state--action pair, run the critic
independently
\(T:=\lceil\log 3/\beta_0\rceil\) times and accept the statement
``unsound'' if any run rejects. If the action is sound, one-sidedness implies
that this procedure never accepts. If the action is unsound, the probability
of seeing no rejection is at most
\((1-\beta_0)^T\leq e^{-\beta_0T}\leq1/3\).
Thus local unsoundness has an RP algorithm. Since local unsoundness is
coNP-complete, this implies \(\mathrm{coNP}\subseteq\mathrm{RP}\).
\end{proof}

The uniformity qualifier is essential here: the reduction requires a single
algorithm that can be run on an arbitrary encoded state--action pair, rather
than an unrelated nonuniform procedure supplied separately for each state.

\subsection{Local-to-Global Amplification}
\label{app:amplification}

\subsubsection{Proof of Theorem~\ref{thm:local}}

\localerrorthm*

\begin{proof}
Fix a reachable valid nonterminal state \(s\). Let \(E_{\mathrm{prop}}\)
denote the event that none of the \(k\) proposed candidates lies in
\(\Sound_x(s)\). By definition,
\(\Pbb(E_{\mathrm{prop}})=\eps_{\mathrm{prop}}(k;s)\).

Condition now on \(E_{\mathrm{prop}}^c\), so at least one progressing-sound
candidate is present. Because the critic is one-sided, every progressing-sound
candidate survives the critic stage. Thus, conditional on proposal coverage,
the only way the local step can select an unsound action is through an
unsound survivor that is not defeated by the sound survivors.

Fix one progressing-sound proposed candidate \(a\) and one unsound proposed
candidate \(b\). Each critic call rejects \(b\) with probability at least
\(\beta\). Conditional on the proposed pool, the \(m\) critic calls are
independent, so
\[
\Pbb(b\textnormal{ survives all critics})
\leq (1-\beta)^m.
\]

Suppose \(b\) survives. For comparator repetition \(j\), let \(Y_j=1\) if the
vote prefers \(a\) to \(b\) and \(Y_j=0\) otherwise. The comparator edge gives
\(\Ebb[Y_j]\geq1/2+\sigma\). The votes are independent across repetitions, and
their randomness is fresh relative to the critic stage. Since \(r\) is odd,
\(a\) fails to win the pairwise majority only if
\(r^{-1}\sum_{j=1}^rY_j\leq1/2\). Hoeffding's inequality therefore gives
\[
\Pbb(a\textnormal{ fails to beat }b)
\leq e^{-2r\sigma^2}.
\]
Freshness of the comparator randomness implies that this bound remains valid
conditional on \(b\)'s survival. Hence
\[
\Pbb\!\left(
b\textnormal{ survives and }a\textnormal{ does not beat }b
\right)
\leq
(1-\beta)^m e^{-2r\sigma^2}.
\]

There are at most \(k^2\) ordered progressing-sound--unsound candidate pairs.
A union bound therefore shows that, except with probability at most
\(k^2(1-\beta)^m e^{-2r\sigma^2}\), every surviving sound candidate defeats
every surviving unsound candidate.

It remains to verify that this event forces a Copeland winner to be sound.
Let \(S\geq1\) and \(U\geq0\) denote the numbers of surviving sound and unsound
candidates. Every sound survivor defeats all \(U\) unsound survivors, so its
Copeland score is at least \(U\). Every unsound survivor loses to all \(S\)
sound survivors and can defeat at most the other \(U-1\) unsound survivors,
so its score is at most \(U-1\). Hence every Copeland maximizer is sound.

Combining this selection-failure event with \(E_{\mathrm{prop}}\) yields
\[
\eps_{\mathrm{loc}}(s)
\leq
\eps_{\mathrm{prop}}(k;s)
+
k^2(1-\beta)^m e^{-2r\sigma^2}.
\]
Using \(1-\beta\leq e^{-\beta}\) gives the second displayed bound in the
theorem.

Finally, if the \(k\) proposer calls are conditionally independent and each
lands in \(\Sound_x(s)\) with probability at least \(\alpha\), then all \(k\)
miss with probability at most
\((1-\alpha)^k\leq e^{-\alpha k}\).
\end{proof}

Two aspects of the proof are worth noting. Independence is needed across
repeated critic calls for a given candidate and across repeated comparator
votes for a given pair, but not across different candidates or different
pairs. Likewise, the Copeland argument does not require comparisons among
sound candidates, or among unsound candidates, to be accurate.


\subsubsection{Proof of Lemma~\ref{lem:cumulative}}

\cumulativelem*

\begin{proof}
For \(t=0,\ldots,L_x-1\), let \(B_t\) be the event that step \(t\) is reached
at a valid nonterminal state and the local step fails, either by returning
failure or by selecting an action outside \(\Sound_x(S_t)\). If step \(t\) is
not reached at such a state, define \(B_t\) to be false.

Suppose none of \(B_0,\ldots,B_{L_x-1}\) occurs. Every action that is actually
applied is then progressing-sound. Consequently, validity is preserved and the
nonnegative integer rank decreases strictly at every applied step. Starting
from rank at most \(L_x\), the process must therefore enter a rank-zero state
within at most \(L_x\) steps. By Definition~\ref{def:vss}, every valid
rank-zero state is terminal, and every terminal output is accepted by \(R_x\).
Hence global failure implies \(\bigcup_{t=0}^{L_x-1}B_t\).

By the union bound,
\(\mathrm{err}_x(k,m,r)\leq\sum_{t<L_x}\Pbb(B_t)\).
Moreover, by the assumed conditional bound and the tower property,
\(\Pbb(B_t)=\Ebb[\Pbb(B_t\mid\mathcal F_t)]
\leq\Ebb[\eps_t]\). Therefore
\[
\mathrm{err}_x(k,m,r)
\leq
\Ebb\!\left[\sum_{t=0}^{L_x-1}\eps_t\right].
\]
If \(\eps_t\leq\eps\) uniformly, the right-hand side is at most
\(L_x\eps\).
\end{proof}

Crucially, no independence across reasoning steps is assumed. The state at
step \(t\), the candidate distribution there, and the local error probability
may all depend arbitrarily on the preceding trajectory; the lemma uses only a
conditional error bound at the state actually reached.


\subsubsection{Proof of Corollary~\ref{cor:global_amplification}}

\begin{proof}
At every reachable valid nonterminal state \(s\),
Theorem~\ref{thm:local} and the uniform assumptions give
\[
\eps_{\mathrm{loc}}(s)
\leq
\eps_{\mathrm{prop}}(k)
+
k^2e^{-\beta m-2r\sigma^2}.
\]
This is a uniform conditional bound on the failure probability of every local
decision encountered along the adaptive trajectory. Applying
Lemma~\ref{lem:cumulative} therefore gives
\[
\mathrm{err}_x(k,m,r)
\leq
L_x\!\left(
\eps_{\mathrm{prop}}(k)
+
k^2e^{-\beta m-2r\sigma^2}
\right).
\]
If the quantity in parentheses is at most \(\delta/L_x\), the right-hand side
is at most \(\delta\), proving the claimed success probability.
\end{proof}


\subsubsection{Allowing Critic False Rejections}

The one-sided critic assumption isolates the main mechanism cleanly, but the
argument extends to critics that occasionally reject good candidates.

\begin{proposition}[Two-sided critic extension]
\label{prop:two_sided_critic}
Suppose that, conditional on the state and candidate, each critic call rejects
a progressing-sound candidate with probability at most \(\gamma\) and rejects
an unsound candidate with probability at least \(\beta\). Under the comparator
and independence conditions of Theorem~\ref{thm:local},
\[
\eps_{\mathrm{loc}}(s)
\leq
\eps_{\mathrm{prop}}(k;s)
+
m\gamma
+
k^2(1-\beta)^m e^{-2r\sigma^2}.
\]
\end{proposition}

\begin{proof}
Condition on the event that the proposal pool contains at least one
progressing-sound candidate. Given the realized proposal pool, choose one such
candidate \(a_\star\) by a deterministic rule, for example the lowest-index
sound candidate. We do not need every sound candidate to survive; it is enough
to ensure that this single witness survives.

Each of the \(m\) critic calls rejects \(a_\star\) with probability at most
\(\gamma\). No independence is needed here: by the union bound, the probability
that at least one call rejects \(a_\star\) is at most \(m\gamma\). Outside this
event, at least one sound candidate survives.

Now fix any sound--unsound proposed pair \((a,b)\). For this pair to threaten
the Copeland selection, both candidates must survive and \(a\) must fail to
beat \(b\). We may upper-bound the sound candidate's survival probability by
one. The unsound candidate survives all \(m\) independent critic calls with
probability at most \((1-\beta)^m\), and conditional on the critic stage the
fresh comparator votes fail with probability at most \(e^{-2r\sigma^2}\).
Thus
\[
\Pbb(a,b\textnormal{ survive and }a\textnormal{ fails to beat }b)
\leq
(1-\beta)^m e^{-2r\sigma^2}.
\]

A union bound over fewer than \(k^2\) sound--unsound pairs shows that, except
with probability given by the last term in the proposition, every surviving
sound candidate defeats every surviving unsound candidate. Provided
\(a_\star\) was not rejected, at least one sound candidate survives, and the
same Copeland-score argument as in Theorem~\ref{thm:local} implies that every
winner is sound. Adding proposal failure and the possible rejection of
\(a_\star\) proves the result.
\end{proof}

The \(m\gamma\) term reveals the basic critic-repetition tradeoff: increasing
\(m\) suppresses unsound survivors exponentially, but, when false rejections
are possible, it also increases the chance of discarding a useful candidate.

\subsection{Blind Spots, Diversity, and Finite-Width Rates}
\label{app:robustness}

\subsubsection{Proof of Lemma~\ref{lem:robust}}

\robustlem*

\begin{proof}
Conditional on \(Z\), each proposal is progressing-sound with probability
\(q_s(Z)\), and the \(k\) proposal calls are independent. Hence all \(k\)
proposals miss with probability \((1-q_s(Z))^k\). Averaging over \(Z\) gives
\(\eps_{\mathrm{prop}}(k;s)=\Ebb[(1-q_s(Z))^k]\).

Split this expectation according to whether \(q_s(Z)=0\) or \(q_s(Z)>0\).
On the first event, the integrand equals one, contributing
\(B_s=\Pbb(q_s(Z)=0)\); the second contribution is exactly \(R_k(s)\).
Thus \(\eps_{\mathrm{prop}}(k;s)=B_s+R_k(s)\).

On the event \(q_s(Z)>0\), \((1-q_s(Z))^k\to0\) pointwise as
\(k\to\infty\), while the integrand is bounded by one. Dominated convergence
therefore gives \(R_k(s)\to0\), and hence
\(\eps_{\mathrm{prop}}(k;s)\to B_s\).
\end{proof}


\subsubsection{Heterogeneous Proposer Portfolios}

\begin{proposition}[Heterogeneous portfolio oracle miss]
\label{prop:heterogeneous_blindspot}
Fix a finite set of proposer families \(g\in\{1,\ldots,G\}\), and suppose
family \(g\) is called \(n_g\) times. Conditional on \(Z\), assume all calls are
independent and a call from family \(g\) is progressing-sound with probability
\(q_g(Z)\). Then
\[
\eps_{\mathrm{prop}}(\{n_g\};s)
=
\Ebb\!\left[
\prod_{g=1}^G(1-q_g(Z))^{n_g}
\right]
\leq
\Ebb[e^{-Q(Z)}],
\]
where \(Q(Z):=\sum_g n_gq_g(Z)\). If
\(Q(Z)\geq\lambda\) outside an event of probability at most \(\eta\), then
\(\eps_{\mathrm{prop}}(\{n_g\};s)\leq\eta+e^{-\lambda}\).
If every \(n_g\to\infty\), the miss probability converges to
\(\Pbb(q_g(Z)=0\text{ for all }g)\).
\end{proposition}

\begin{proof}
Conditional on \(Z\), the probability that every call from family \(g\) misses
is \((1-q_g(Z))^{n_g}\). Conditional independence across calls therefore gives
the product formula. Applying \(1-u\leq e^{-u}\) to every factor yields the
pointwise bound
\[
\prod_g(1-q_g(Z))^{n_g}
\leq
\exp\!\left(-\sum_g n_gq_g(Z)\right)
=e^{-Q(Z)}.
\]

Let \(A=\{Q(Z)\geq\lambda\}\). On \(A\), the product is at most
\(e^{-\lambda}\); on \(A^c\), it is at most one. Since
\(\Pbb(A^c)\leq\eta\), its expectation is at most
\(e^{-\lambda}+\eta\).

Finally, suppose every \(n_g\to\infty\). For a fixed \(Z\), if
\(q_g(Z)=0\) for every family, the product equals one for all sample sizes. If
at least one family has \(q_g(Z)>0\), the corresponding factor tends to zero,
so the entire product tends to zero. The product is bounded by one, and
dominated convergence gives the claimed limit.
\end{proof}

Thus proposer diversity affects the asymptotic ceiling only when the families
have complementary support. Additional samples from one family reduce
finite-width error on cases it can already solve; they cannot repair cases on
which that family has zero useful-action mass.


\subsubsection{Finite-Width Rates}

\begin{proposition}[Lower-tail bound]
\label{prop:lower_tail_rate}
Suppose that for some \(a,C>0\),
\(\Pbb(0<q_s(Z)\leq t)\leq Ct^a\) for all sufficiently small \(t\). Then,
for all sufficiently large \(k\),
\[
R_k(s)
\leq
C\left(\frac{a\log k}{k}\right)^a+k^{-a}.
\]
If instead \(q_s(Z)\geq\tau>0\) whenever \(q_s(Z)>0\), then
\(R_k(s)\leq e^{-\tau k}\).
\end{proposition}

\begin{proof}
For the second statement, on the event \(q_s(Z)>0\) we have
\(q_s(Z)\geq\tau\), and therefore
\((1-q_s(Z))^k\leq e^{-kq_s(Z)}\leq e^{-\tau k}\).
Averaging gives the claimed bound.

For the first statement, set \(t_k:=a\log(k)/k\). For sufficiently large
\(k\), \(t_k\) lies in the range on which the assumed lower-tail estimate
holds. Split \(R_k(s)\) according to the events
\(0<q_s(Z)\leq t_k\) and \(q_s(Z)>t_k\). On the first event, the integrand is
at most one, so its contribution is at most \(Ct_k^a\). On the second,
\((1-q_s(Z))^k\leq e^{-kq_s(Z)}\leq e^{-kt_k}=k^{-a}\).
The second contribution is therefore at most \(k^{-a}\). Adding the two
pieces yields the result.
\end{proof}

The result separates the limiting blind-spot floor from the rate at which the
finite-width residual approaches it. Even when \(B_s=0\), substantial mass
near \(q_s(Z)=0\) can make convergence with \(k\) slow.


\subsubsection{Terminal Oracle Curves}

For completeness, the task-level oracle formula in
Section~\ref{sec:blindspots} follows by the same conditioning argument. Given
\(X=x\), \(k\) independent samples from \(P(\cdot\mid x)\) all fail with
probability \((1-q_P(x))^k\). Averaging over \(X\) gives
\(p_{\mathrm{oracle}}(k;P)=
\Ebb[1-(1-q_P(X))^k]\).
For every \(x\), the integrand converges to
\(\mathbf 1\{q_P(x)>0\}\), and it is bounded by one. Dominated convergence
therefore yields
\(\lim_{k\to\infty}p_{\mathrm{oracle}}(k;P)
=1-\Pbb(q_P(X)=0)\).

This also makes clear why finite-\(k\) oracle failure should not be identified
with the asymptotic blind-spot mass: at finite width it contains the residual
sampling term as well.

\section{Polynomial-Resource Consequence}
\label{app:poly_resource}

Corollary~\ref{cor:global_amplification} gives an efficient amplification
regime whenever coverage, local identifiability, and trajectory length remain
polynomially controlled.

\begin{corollary}[Polynomial-resource amplification]
\label{cor:poly_resource}
Suppose that, for size-\(N\) inputs, \(L_x\), \(|P_N|\),
\(1/\alpha_0\), \(1/\beta_0\), and \(1/\sigma_0\) are polynomially bounded,
and that each proposer, critic, comparator, and verifier call runs in
polynomial time. Then for every \(\delta\in(0,1)\), there exist
\(k,m,r=\poly(N,\log(1/\delta))\) such that
\(\Pbb(R_x(\Pi_{k,m,r}(x))=0)\leq\delta\). The all-pairs implementation uses
\(O(L_x(k+mk+rk^2))\) role calls.
\end{corollary}

\begin{proof}
If \(L_x=0\), the initial state is already terminal and the claim is trivial,
so assume \(L_x\geq1\).

By Theorem~\ref{thm:bridge}, it suffices to choose
\[
k\geq
|P_N|
\left\lceil
\alpha_0^{-1}\log\frac{2L_x}{\delta}
\right\rceil .
\]
Then the proposal-failure probability at every reachable valid nonterminal
state is at most \(\delta/(2L_x)\).

It remains to make the selection term at most the same amount. It suffices that
\(\beta_0m+2\sigma_0^2r
\geq\log(2k^2L_x/\delta)\). One convenient allocation is
\[
m\geq
\left\lceil
\frac{1}{2\beta_0}
\log\frac{2k^2L_x}{\delta}
\right\rceil,
\qquad
r\geq
\left\lceil
\frac{1}{4\sigma_0^2}
\log\frac{2k^2L_x}{\delta}
\right\rceil.
\]
With these choices,
\(k^2e^{-\beta_0m-2\sigma_0^2r}\leq\delta/(2L_x)\).
Corollary~\ref{cor:global_amplification} therefore gives total failure
probability at most \(\delta\).

At each reasoning step, the implementation makes \(k\) proposer calls and at
most \(mk\) critic calls. At most \(k\) candidates survive, so there are at
most \(\binom{k}{2}\) compared pairs and hence at most
\(r\binom{k}{2}\) comparator calls. There are at most \(L_x\) steps, giving
\(O(L_x(k+mk+rk^2))\) role calls overall.

Finally, the assumed polynomial bounds on \(L_x\), \(|P_N|\),
\(1/\alpha_0\), \(1/\beta_0\), and \(1/\sigma_0\) imply that the displayed
choices of \(k,m,r\) are polynomial in \(N\) and \(\log(1/\delta)\).
Polynomial per-call running time then gives polynomial overall running time.
\end{proof}

The assumptions are substantive. Efficient amplification can fail because the
useful proposal mass is exponentially small, local identification edges vanish
too rapidly, the required trajectory is superpolynomially long, or the local
checking procedures are themselves computationally expensive. Efficient
verification of completed outputs alone rules out none of these obstructions.

\section{SWE-bench Implementation Details}
\label{app:swebench_implementation}

The SWE-bench experiments are built around a fixed-pool design. Proposal
generation is completed first; the same cached patches are then reused across
selector variants. This keeps proposal coverage fixed when we compare critics,
comparators, gates, and aggregation rules.

\subsection{Benchmark and proposal pools}

SWE-bench Verified~\citep{Jimenez2024SWEbench} contains \(500\)
repository-level software-engineering tasks. Each task provides a repository
state and an issue description, and a patch is correct only if it passes the
benchmark's held-out tests. Hidden-test outcomes are used for final scoring and
oracle diagnostics only; they are never exposed to the selector.

For each task we generate \(k=8\) candidate patches with independent
\texttt{GPT-5.4 nano} Mini-SWE-agent trajectories. Each trajectory receives the
issue, interacts with the repository, and returns a patch. The eight patches
are cached before selector ablations begin.

Oracle best-of-\(k\) is computed after hidden-test evaluation. A task is oracle
solved if at least one of the \(k\) cached patches passes the hidden tests. It
therefore measures proposal coverage, not a deployable selection rule.

\subsection{Selection pipeline}

{\bf Binary critic.\enspace}
The binary patch judge receives the issue and one candidate patch and predicts
whether the patch resolves the issue without breaking previously working
behavior. Each response contains a binary decision, a short rationale, and a
confidence score. Multiple independent judgments are cached for each patch.
The basic critic score is the number of positive votes; the
confidence-weighted variant sums confidence over positive votes. Ties are
broken by proposal index.

{\bf Pairwise comparator.\enspace}
The comparator receives the issue and two candidate patches. It states a
failure hypothesis, summarizes the relevant changes in each patch, checks
whether those changes address the hypothesized failure, and records possible
collateral changes. It then outputs \(A\), \(B\), or \(\mathrm{TIE}\), together
with a confidence score.

For each unordered pair \((p_i,p_j)\), we query the comparator in both
presentation orders. A pair is counted as a win for \(p_i\) only if both orders,
after mapping positions back to patch identities, select \(p_i\); likewise for
\(p_j\). Disagreement, a tie, or a missing response is treated as a tie. This
conservative rule removes direct lead-position wins from the tournament.

Pairwise outcomes are aggregated by Copeland score: a pairwise win contributes
one point, and a tie contributes no point to either patch. The patch with the
largest score is selected; remaining ties are broken by proposal index.

{\bf Critic-gated comparator.\enspace}
The combined selector first filters patches by critic support and then runs the
pairwise tournament on the survivors. If no patch passes the gate, it falls
back to the highest-scoring patch under the critic. Unless otherwise stated,
the gated-comparator ablations use yes-rate threshold \(\tau=0.5\).
Threshold sweeps are offline diagnostics.

All selector prompts require strict JSON. A binary-judge parse failure is an
abstention and contributes no positive vote. A missing or unparseable
pairwise response is treated as a tie.

\subsection{Evaluation and ablations}

We vary proposal and judging budgets separately. Proposal-budget curves restrict
the cached eight-patch pool to subsets of size \(k\) and recompute both oracle
and selector performance. When several subsets of a given size are available,
we average over them rather than privilege a proposal ordering. Vote-budget
curves reuse the same proposal pool while varying the number of critic votes
per patch or comparator votes per pair.

For the repository-level failure decomposition, each task falls into one of
three disjoint groups:
\begin{itemize}[leftmargin=2em,nosep]
\item \textbf{solved:} the selected patch passes the hidden tests;
\item \textbf{oracle reachable, missed:} some cached patch passes, but the
selector does not choose a passing patch; or
\item \textbf{oracle unreachable:} none of the cached patches passes.
\end{itemize}
The second group is selection error on the fixed pool; the third is proposal
failure at that budget.

The main metric is hidden-test solve rate. We also report oracle best-of-\(k\)
and the oracle-gap recovery ratio from Section~\ref{sec:blindspots}. For the
fixed-pool analyses, \(p_1\) is the mean success rate of a uniformly chosen
proposal from the pool; this is distinct from the separate \(67.0\%\) headline
single-run baseline.

\subsection{Selector prompts}
\label{app:prompts_tiebreaking}

The issue text is truncated to \(2000\) characters and each patch to \(8000\)
characters before judging.

{\bf Binary patch-judge prompt.\enspace}
\begin{promptbox}[Binary patch-judge prompt]
You are evaluating a single patch against the GitHub issue it claims to fix.
Output yes if and only if running the patched code passes the failing test
described in the issue and does not break any test that previously passed.

Trace the failure path from the issue to the originating code. Check whether
the patch modifies that code path. Then state whether the patch produces the
expected behavior on the smallest input that exhibits the failure.

ISSUE:
{problem_statement}

PROPOSED PATCH:
{patch}

Respond with EXACTLY this JSON, no markdown fences, no extra text:
{
  "resolves": <true|false>,
  "reasoning": "<one sentence: smallest input under original code produces X; under patch produces Y>",
  "confidence": <integer 1-5>
}
\end{promptbox}

Malformed responses, refusals, or parse failures are treated as abstentions.
The default binary score is the number of positive votes,
\(s_{\mathrm{yes}}(p)=\#\{v:v.\mathrm{resolves}=\mathrm{true}\}\).
The confidence-weighted ablation instead uses the sum of reported confidence
over positive votes.

{\bf Pairwise comparator prompt.\enspace}
\begin{promptbox}[Pairwise comparator prompt]
You are deciding which of two patches is more likely to make the failing
tests pass without breaking the passing tests. That is the only question
you are answering. Aesthetics, ambition, and visible effort are not relevant.

Do NOT select a patch because it makes more changes, addresses more cases,
or appears more thorough. The patch that minimally resolves the stated
issue is preferred. A 4kb patch that fixes the bug beats an 11kb patch
that doesn't.

ISSUE:
{problem_statement}

PATCH A:
{patch_a}

PATCH B:
{patch_b}

REQUIRED STRUCTURAL COMPARISON (fill these in BEFORE deciding):

  1. Failing test hypothesis: state the smallest hypothesis about what is
     wrong, derived from the issue text alone. One sentence. This is the
     ground truth you compare both patches against.

  2. A_changes: list the specific lines/functions Patch A modifies.
  3. B_changes: list the specific lines/functions Patch B modifies.

  4. A_consistent_with_hypothesis: do A's changes plausibly cause the
     failing test in the issue to start passing? (true/false + one-line
     justification)
  5. B_consistent_with_hypothesis: same question for B.

  6. A_collateral: does A change behavior on inputs unrelated to the
     failure mode? (true/false + one-line justification)
  7. B_collateral: same question for B.

The decision falls out of this comparison; do not pull a winner from prior.
If exactly one patch is consistent with the hypothesis and the other is
not, that one wins. If both are consistent, prefer the one with less
collateral. If both fail the hypothesis, output TIE. If they are
functionally equivalent (same lines changed differently, or different
lines that produce the same behavior), output TIE.

Respond in this EXACT JSON format -- JSON only, no markdown fences, no extra text:
{
  "hypothesis": "<one sentence>",
  "a_changes": "<files/functions modified by A>",
  "b_changes": "<files/functions modified by B>",
  "a_consistent": <true|false>,
  "b_consistent": <true|false>,
  "a_collateral": <true|false>,
  "b_collateral": <true|false>,
  "winner": "A" | "B" | "TIE",
  "confidence": <integer 1-5>,
  "reasoning": "<one sentence: why the winner falls out of the structural comparison>"
}
\end{promptbox}

If a pairwise response remains malformed after the allowed retries, it is
treated as missing and therefore as a tie under the conservative aggregation
rule above.

\subsection{Selector ablations}
\label{app:selector_ablations}

Table~\ref{tab:selector_ablations} uses the same cached \(k=8\)
\texttt{GPT-5.4 nano} pool throughout. The tournament block varies the
aggregation rule with five comparator calls per matchup. Copeland round-robin
and strict dominance reach \(75.8\%\), while single elimination reaches
\(74.6\%\). The threshold block shows a different effect: a permissive critic
gate improves on pure Copeland, while stricter gates begin to remove useful
patches. This is consistent with the two-sided-critic extension in
Proposition~\ref{prop:two_sided_critic}: repeated rejection is useful only while
false rejection remains controlled.

\begin{table}[t]
\centering
\small
\caption{
\textbf{Selector ablations on SWE-bench Verified.}
All rows use the same \(k=8\) \texttt{GPT-5.4 nano} proposal pool. Tournament rows vary
the \texttt{GPT-5.4 nano} comparator aggregation rule with \(5\) comparator calls per
pairwise matchup. Threshold rows vary the five-critic gate with \(5\) critics
and \(5\) comparators. Oracle-gap recovery is computed using the metric defined
in Section~\ref{sec:blindspots}.
}
\label{tab:selector_ablations}
\begin{adjustbox}{max width=\linewidth}
\begin{tabular}{llcc}
\toprule
Ablation & Setting & Resolve rate & Oracle-gap recovery \\
\midrule
Tournament rule & Copeland round-robin & \(75.8\%\) & \(73.3\%\) \\
Tournament rule & Sequential king & \(75.6\%\) & \(71.7\%\) \\
Tournament rule & Strict dominance & \(75.8\%\) & \(73.3\%\) \\
Tournament rule & Single-elim bracket & \(74.6\%\) & \(63.3\%\) \\
\midrule
Critic threshold & \(\tau \geq 0\) no gate & \(75.8\%\) & \(73.3\%\) \\
Critic threshold & \(\tau \geq 1\), drop 0-yes patches & \(\mathbf{76.4\%}\) & \(\mathbf{78.3\%}\) \\
Critic threshold & \(\tau \geq 2\), tied with \(\tau\geq1\) & \(\mathbf{76.4\%}\) & \(\mathbf{78.3\%}\) \\
Critic threshold & \(\tau \geq 3\), starts losing 1 instance & \(76.2\%\) & \(76.7\%\) \\
Critic threshold & \(\tau \geq 4\) & \(76.2\%\) & \(76.7\%\) \\
Critic threshold & \(\tau \geq 5\), unanimous gate & \(76.0\%\) & \(75.0\%\) \\
\bottomrule
\end{tabular}
\end{adjustbox}
\end{table}

\section{Additional Empirical Results}
\label{app:additional_empirical}

This section provides the detailed empirical evidence behind the main-text
scaling, identifiability, and resource results. We first standardize the
fixed-pool protocol, then report cross-model scaling, controlled proposal-width
experiments, simple selection controls, judge diagnostics, resource accounting,
and additional boundary cases.

\subsection{Protocol Notes}
All additional experiments use fixed proposal pools: proposals are generated first, hidden outcomes are withheld from the selector, and hidden outcomes are used only afterward for scoring and oracle diagnostics. The cross-model SWE-bench comparison uses the same three-vote comparator-only replay for all four pools; the headline nano result remains the separate full critic--comparator run described in the main text.

For MATH L4--5 and GSM-Plus, Qwen2.5 proposers are judged by Qwen3-1.7B. LiveCodeBench uses Qwen3.6-35B-A3B proposals and pairwise judging, with public tests as a free control. miniF2F uses Goedel-Prover-V2-8B whole-proof proposals and Qwen3.6-35B-A3B selection without Lean outcomes. The additional Qwen3.6 tasks reuse saved five-proposal pools across first-proposal, plurality, critic-only, comparator-only, and full-committee policies.

\subsection{Cross-Model Scaling}

The cross-model study asks whether the SWE-bench pattern is specific to the
headline nano pool. Figure~\ref{fig:cross_model_scaling} applies the same
three-vote comparator-only selector to four fixed proposal pools. Oracle
coverage rises with proposal width in every pool, but the fraction recovered by
selection differs substantially: at \(k=8\), recovery is \(69\%\),
\(46\%\), \(58\%\), and \(27\%\) for nano, Qwen3.6-27B,
Qwen3.6-35B-A3B, and Gemma-4-31B, respectively. Table~\ref{tab:decomposition}
reports the corresponding endpoints together with the cross-domain results used
in Figure~\ref{fig:identifiability}a.

\begin{figure}[t]
\centering
\includegraphics{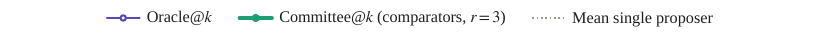}\\[0.1em]
\begin{tabular}{@{}c@{}c@{}c@{}c@{}}
\includegraphics{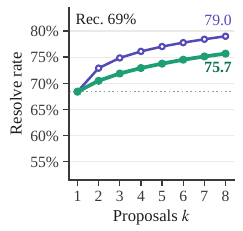} &
\includegraphics{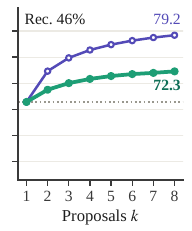} &
\includegraphics{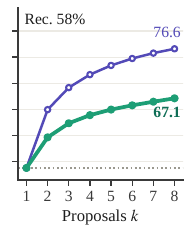} &
\includegraphics{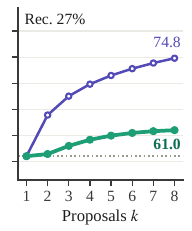} \\
\hspace{0.40in}{\footnotesize (a) \texttt{GPT-5.4 nano}} &
{\footnotesize (b) Qwen3.6-27B} &
{\footnotesize (c) Qwen3.6-35B-A3B} &
{\footnotesize (d) Gemma-4-31B}
\end{tabular}
\vspace{-0.3em}
\caption{\textbf{Cross-model scaling on SWE-bench Verified.}
Oracle and common comparator-only selection as the proposal budget grows for
four proposer models. Dotted lines show mean single-proposal performance.}
\label{fig:cross_model_scaling}
\end{figure}

\begin{table}[t]
\centering
\small
\setlength{\tabcolsep}{3.5pt}
\caption{\textbf{Fixed-pool transfer results across tasks and proposer families.}
SWE-bench committee values use the common \(r=3\) comparator-only replay;
the headline nano result is the separate \(76.4\%\) critic--comparator run.}
\label{tab:decomposition}
\begin{adjustbox}{max width=\linewidth}
\begin{tabular}{@{}llrrccccc@{}}
\toprule
Benchmark & Proposer & \(n\) & \(k\) & Pass@1 & Free ctrl. & Committee & Oracle@\(k\) & Rec. \\
\midrule
\multicolumn{9}{@{}l}{\emph{SWE-bench Verified}} \\
 & GPT-5.4 nano & 500 & 8 & 68.4 & 70.0$^{a}$ & 75.7 & 79.0 & 69\% \\
 & Qwen3.6-27B & 500 & 8 & 66.4 & 67.4$^{a}$ & 72.3 & 79.2 & 46\% \\
 & Qwen3.6-35B-A3B & 500 & 8 & 53.8 & 57.4$^{a}$ & 67.1 & 76.6 & 58\% \\
 & Gemma-4-31B & 500 & 8 & 56.0 & 62.6$^{a}$ & 61.0 & 74.8 & 27\% \\
\midrule
MATH L4--5 & Qwen2.5-3B & 500 & 8 & 48.8 & 58.6$^{b}$ & 71.2 & 75.0 & 85\% \\
GSM-Plus & Qwen2.5-1.5B & 1,319 & 8 & 49.1 & 66.9$^{b}$ & 75.5 & 83.6 & 77\% \\
LiveCodeBench & Qwen3.6-35B-A3B & 258 & 5 & 70.2 & 75.2$^{c}$ & 75.2 & 81.0 & 46\% \\
miniF2F-hard & Goedel-Prover-V2-8B & 164 & 5 & 19.5 & -- & 25.0 & 32.3 & 43\% \\
\bottomrule
\end{tabular}
\end{adjustbox}
\\[-0.2em]
{\scriptsize $^{a}$String-similarity consensus; $^{b}$majority vote; $^{c}$public-test filter.}
\end{table}

\subsection{Controlled GSM-Plus Scaling}

To isolate proposal width from changes in the proposer or judge, we subsample a
single saved eight-proposal GSM-Plus pool and average over subsets of each size.
The oracle rises from \(49.1\%\) at \(k=1\) to \(83.6\%\) at
\(k=8\). Pairwise selection tracks this increase while recovering
\(76.7\%\)--\(79.5\%\) of the oracle-exposed gain for
\(k=2,4,8\). In contrast, the pointwise critic increasingly loses
resolution: its top-score tie rate rises from \(55.6\%\) to \(92.6\%\).
Table~\ref{tab:gsm_scaling} gives the exact values behind
Figure~\ref{fig:component_scaling}c.

\begin{table}[t]
\centering
\small
\setlength{\tabcolsep}{5pt}
\caption{\textbf{Controlled proposal-width sweep on GSM-Plus.}
All rows are computed from the same saved eight-proposal pool by subset averaging.
Pairwise selection recovers a nearly constant fraction of the oracle-exposed gain
as \(k\) grows, while the pointwise critic increasingly ties candidates at its
maximum score. At \(k=8\), the bootstrap 95\% interval for pairwise oracle-gap
recovery is \([0.727,0.807]\). Critic statistics use 1,318 tasks because one
critic response is missing; other columns use 1,319 tasks.}
\label{tab:gsm_scaling}
\begin{tabular}{@{}rrrrrrr@{}}
\toprule
\(k\) & Oracle & Majority & Critic & Pairwise & Pairwise Rec. & Critic top-tie \\
\midrule
1 & 49.1 & 49.1 & --   & 49.1 & --    & -- \\
2 & 63.6 & 49.6 & 57.0 & 60.2 & 76.7\% & 55.6\% \\
4 & 75.0 & 59.7 & 61.2 & 69.7 & 79.5\% & 80.4\% \\
8 & 83.6 & 66.9 & 62.8 & 75.5 & 76.7\% & 92.6\% \\
\bottomrule
\end{tabular}
\end{table}

\subsection{Selection Controls}

Simple heuristics provide a useful check on whether learned judging is needed.
Their behavior is not stable across proposal pools: for example, shortest-patch
selection reaches \(62.8\%\) on Gemma but only \(29.4\%\) on the
Qwen3.6-35B-A3B pool. Similarity consensus is the only zero-cost rule here that
exceeds random selection on all four pools, but the learned comparator is much
stronger on three of them. The Gemma exception is therefore a genuine boundary
case rather than evidence that one surface heuristic is generally sufficient.

\begin{table}[t]
\centering
\small
\caption{\textbf{Simple controls on the fixed \(k=8\) SWE-bench pools.}
The committee and oracle are included as references; the remaining rows require
no learned judge.}
\label{tab:heuristics}
\begin{tabular}{@{}lcccc@{}}
\toprule
Selection rule & Nano & Qwen-27B & Qwen-A3B & Gemma-31B \\
\midrule
Random pick & 68.4 & 66.4 & 53.8 & 56.0 \\
First proposal & 68.4 & 68.6 & 54.8 & 61.8 \\
Shortest patch & 69.4 & 65.4 & 29.4 & \textbf{62.8} \\
Fewest files & 69.4 & 65.0 & 26.4 & 62.2 \\
Similarity consensus & 70.0 & 67.4 & 57.4 & 62.6 \\
\midrule
Committee (\(r=3\)) & \textbf{75.7} & \textbf{72.3} & \textbf{67.1} & 61.0 \\
Oracle@8 & 79.0 & 79.2 & 76.6 & 74.8 \\
\bottomrule
\end{tabular}
\end{table}

Aggregation choice is secondary on the corrected Qwen pools: Copeland,
forced-margin, confidence-weighted, and anti-tie variants differ by at most one
percentage point. The larger variation comes from the quality of the available
identification signal, not the tournament rule itself.

\subsection{Detailed Identifiability Diagnostics}

Table~\ref{tab:edges_detailed} connects observed judge behavior to the
identifiability quantities used in the theory and main text. MATH and GSM-Plus
have large empirical comparator edges, \(\hat\sigma=0.40\) and
\(0.38\), and correspondingly high oracle-gap recovery. LiveCodeBench has a
smaller edge and lower recovery. Cross-task comparisons are descriptive, since
the proposers, judges, and tasks differ. The clean intervention is the final
pair of SWE rows: on the same Qwen3.6-27B proposal pool, repository access raises
\(\hat\sigma\) from \(0.04\) to \(0.09\) and recovery from
\(24\%\) to \(40\%\).

\begin{table}[t]
\centering
\scriptsize
\caption{\textbf{Empirical identifiability diagnostics.}
TPR is critic acceptance on correct candidates; \(\hat\beta\) is critic
rejection rate on incorrect candidates; \(\tau\) is comparator tie rate;
\(a\) is accuracy conditional on a decisive comparison; and swap agr. is
agreement across presentation orders. The empirical comparator edge is
\(\hat\sigma=(1-\tau)(a-\tfrac12)\), as defined in Section~\ref{subsec:identifiability_empirical}.}
\label{tab:edges_detailed}
\begin{adjustbox}{max width=\linewidth}
\begin{tabular}{@{}llccccccc@{}}
\toprule
Setting & Judge & TPR & \(\hat\beta\) & \(\tau\) & \(a\) & swap agr. & \(\hat\sigma\) & Rec. \\
\midrule
MATH L4--5 & Qwen3-1.7B & 0.94 & 0.69 & 0.15 & 0.98 & 0.81 & 0.40 & 85\% \\
GSM-Plus & Qwen3-1.7B & 0.96 & 0.65 & 0.18 & 0.97 & 0.85 & 0.38 & 77\% \\
LiveCodeBench & Qwen3.6-35B-A3B & 0.97 & 0.44 & 0.35 & 0.78 & 0.70 & 0.18 & 46\% \\
\midrule
SWE critic & Qwen3.6-27B & 0.94 & 0.04 & -- & -- & -- & -- & -- \\
SWE comparator, patch only & Qwen3.6-27B & -- & -- & 0.32 & 0.56 & 0.74 & 0.04 & 24\% \\
SWE comparator, repo-aware & Qwen3.6-27B & -- & -- & 0.33 & 0.63 & 0.74 & 0.09 & 40\% \\
\bottomrule
\end{tabular}
\end{adjustbox}
\end{table}

\subsection{Cost and Token Accounting}

The resource accounting separates two questions: how much extra selection
budget is useful on a fixed pool, and how large the judging stage is relative to
proposal generation. On the nano SWE-bench pool, one comparator vote per pair
reaches \(75.48\%\) at \(\$0.969\) per task and three votes reach
\(76.16\%\) at \(\$1.104\); substantially larger comparison budgets do
not improve accuracy. The token split is task-dependent, however: judging is
only \(2.8\%\) of logical model tokens in the reported SWE configuration but
\(98\%\) on MATH and GSM-Plus. Thus the inexpensive-selector observation is
specific to the SWE deployment rather than a universal property of
orchestration.

\begin{table}[t]
\centering
\scriptsize
\caption{\textbf{Selection budget on the fixed \(k=8\) nano pool.}
Costs are modeled from measured role-level token profiles.}
\label{tab:budget}
\begin{adjustbox}{max width=\linewidth}
\begin{tabular}{@{}cclrrcrrcc@{}}
\toprule
& & & \multicolumn{2}{c}{Calls} & Judge tok. (k) & \multicolumn{2}{c}{\$ / task} & & \\
\cmidrule(lr){4-5}\cmidrule(lr){7-8}
\(m\) & \(r\) & Aggregation & critic & comp. & in / out & judge & total & Resolved & Rec. \\
\midrule
-- & -- & random & 0 & 0 & 0 / 0 & 0.000 & 0.899 & 68.40 & 0.0 \\
0 & 3 & top-3 pruned & 0 & 9 & 16 / 18 & 0.023 & 0.922 & 72.47 & 38.4 \\
0 & 3 & top-4 pruned & 0 & 17 & 30 / 35 & 0.046 & 0.945 & 73.92 & 52.1 \\
0 & 3 & bracket & 0 & 20 & 35 / 41 & 0.052 & 0.951 & 75.20 & 64.2 \\
0 & 1 & all pairs & 0 & 26 & 46 / 53 & 0.070 & 0.969 & 75.48 & 66.8 \\
0 & 3 & all pairs, early stop & 0 & 78 & 138 / 160 & 0.203 & 1.102 & 76.16 & 73.2 \\
0 & 3 & all pairs & 0 & 78 & 138 / 160 & 0.205 & 1.104 & 76.16 & 73.2 \\
5 & 5 & bracket & 38 & 30 & 107 / 193 & 0.246 & 1.145 & 75.80 & 69.8 \\
0 & 5 & all pairs & 0 & 131 & 232 / 268 & 0.341 & 1.240 & 76.19 & 73.5 \\
5 & 3 & all pairs & 38 & 69 & 176 / 273 & 0.348 & 1.247 & 76.36 & 75.1 \\
5 & 5 & all pairs & 38 & 115 & 257 / 367 & 0.467 & 1.366 & 76.13 & 73.0 \\
0 & 15 & all pairs & 0 & 368 & 652 / 753 & 0.955 & 1.854 & 75.89 & 70.7 \\
\bottomrule
\end{tabular}
\end{adjustbox}
\end{table}

\begin{table}[t]
\centering
\small
\caption{\textbf{Per-task model-token usage} (thousands) for proposal generation
and judging. Judge share is the fraction of total logical model tokens
attributable to the judging stage.}
\label{tab:tokens}
\begin{adjustbox}{max width=\linewidth}
\begin{tabular}{@{}lrcrrrrr@{}}
\toprule
Experiment & \(n\) & \(k\) & Prop. in & Prop. out & Judge in & Judge out & Judge share \\
\midrule
SWE-bench (nano, \(m=5,r=3\)) & 500 & 8 & 15,102.2 & 400.6 & 175.8 & 272.5 & 2.8\% \\
MATH L4--5 & 500 & 8 & 1.1 & 5.7 & 136.1 & 194.8 & 98\% \\
GSM-Plus & 1,319 & 8 & 1.0 & 2.6 & 80.7 & 136.6 & 98\% \\
LiveCodeBench v5--v6 & 258 & 5 & 3.3 & 135.1 & 59.1 & 348.9 & 75\% \\
miniF2F-hard & 164 & 5 & -- & 41.6 & -- & 97.7 & 70\% \\
Blocksworld & 598 & 5 & 2.4 & 27.8 & 6.6 & 153.3 & 84\% \\
\bottomrule
\end{tabular}
\end{adjustbox}
\end{table}

\subsection{Additional Tasks and Verifier Boundary Cases}

These experiments probe cases in which learned selection is either unnecessary
or weakly aligned with the task. Blocksworld is the clearest verifier-backed
boundary: a public plan checker reaches the same \(99.0\%\) success as the
learned committee on the saved five-proposal pool while using substantially
fewer model tokens. The remaining tasks show smaller gains, and CommonGen gives
a negative control in which comparator-based selection reduces success.

\begin{figure}[t]
\centering
\begin{minipage}[c]{0.49\linewidth}
  \includegraphics[width=\linewidth]{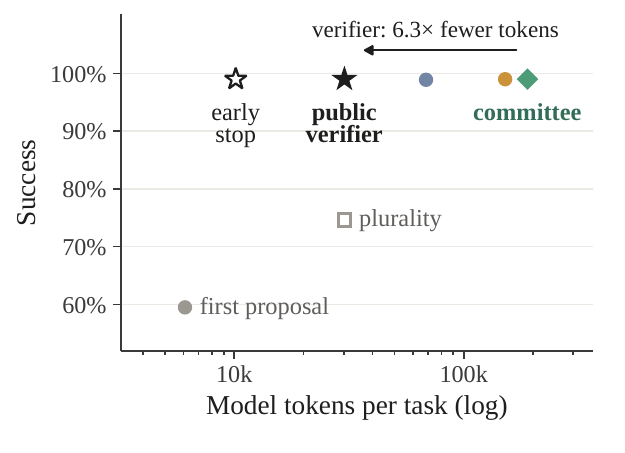}
\end{minipage}\hfill
\begin{minipage}[c]{0.47\linewidth}
\caption{\textbf{Direct verification on Blocksworld.}
A public plan checker matches the learned selector on the same five-proposal
pool while requiring substantially fewer model tokens.}
\label{fig:blocksworld_cost}
\end{minipage}
\end{figure}

On the 598 completed Blocksworld tasks, the first proposal solves \(59.5\%\)
and both the committee and public verifier reach \(99.0\%\). On all 600
intended tasks, counting two task errors as failures, the corresponding rates
are \(59.3\%\) and \(98.7\%\). The full committee uses \(113.70\)M logical
model tokens, while checking the five proposals with the public verifier uses
\(18.10\)M proposal tokens; stopping after the first verified plan would use
\(6.08\)M proposal tokens on the saved draws. When a cheap sound verifier
exists, it should be preferred to learned judging.

\begin{table}[t]
\centering
\small
\caption{\textbf{Selection policies on additional Qwen3.6-35B-A3B tasks.}
Success is reported on completed selections (\%). \(\Delta\) is the full
committee minus the first-proposal baseline, computed from unrounded scores.
Tokens reports full-committee model-token use relative to the first-proposal
baseline.}
\label{tab:additional}
\begin{adjustbox}{max width=\linewidth}
\begin{tabular}{@{}lrcccccrcr@{}}
\toprule
Task & \(n\) & First & Plurality & Critics & Comparators & Committee & \(\Delta\) & Verifier & Tokens \\
\midrule
Blocksworld & 598 & 59.5 & 74.6 & 98.8 & \textbf{99.0} & \textbf{99.0} & +39.5 & 99.0 & 31\(\times\) \\
IFEval & 532 & 88.9 & 88.9 & \textbf{91.4} & 91.0 & 91.0 & +2.1 & -- & 26\(\times\) \\
SocialIQA & 1,954 & 82.5 & 83.1 & 82.9 & 83.4 & \textbf{83.5} & +1.0 & -- & 12\(\times\) \\
AlphaNLI & 1,532 & 86.2 & \textbf{87.0} & 86.3 & 86.4 & 86.3 & +0.1 & -- & 11\(\times\) \\
CommonGen & 992 & 80.5 & 80.5 & \textbf{80.7} & 72.7 & 73.0 & -7.6 & -- & 28\(\times\) \\
\bottomrule
\end{tabular}
\end{adjustbox}
\end{table}

CommonGen is a useful negative control: comparator-based selection is
misaligned with exact required-concept coverage and reduces success from
\(80.5\%\) to \(73.0\%\). Orchestration is therefore not generically monotone;
its value depends on the alignment of the local identification signal with the
terminal objective.

\subsection{Preliminary IFBench Result}
On the 220 prompts with complete selections, loose prompt-level success changes
from \(29.1\%\) to \(31.4\%\) with oracle@5 \(45.5\%\); strict prompt-level
success is unchanged at \(22.3\%\) with oracle@5 \(37.3\%\). Because the
intended 300-prompt run is incomplete, we report these values only as
preliminary diagnostics and exclude them from the main empirical claims.

\section{Extended Literature Review}
\label{app:extended-related}

The literature on inference-time reasoning spans several traditions that are
close to our setting. We organize them around two questions: whether additional
computation increases the chance of producing a useful candidate, and whether
the system can identify that candidate once it appears.

{\bf Boosting and weak-to-strong amplification.\enspace}
Classical boosting combines weak predictive signals into a stronger predictor
under supervised feedback
\citep{Schapire1990,Freund1995Boosting,FreundSchapire1997}, while bagging,
stacking, and related ensemble methods combine multiple predictors for variance
reduction or robustness
\citep{Breiman1996Bagging,Wolpert1992Stacked,Dietterich2000Ensemble}.
Several recent works bring related ideas to language models.
\citep{Manikandan2023} use language models as weak learners in a supervised
boosting pipeline; weak-to-strong generalization studies whether weak
supervision can elicit stronger behavior from a more capable model
\citep{Burns2024}; and iterated amplification uses decomposition to let weak
experts supervise harder problems \citep{Christiano2018}. LLMBoost
\citep{Chen2025LLMBoost} combines models through training and representation
sharing. Our setting keeps model weights fixed and studies amplification
entirely at inference time.

{\bf Sampling and test-time scaling.\enspace}
Self-consistency, rationale ensembles, and Universal Self-Consistency improve
reasoning by drawing several outputs and aggregating them
\citep{Wang2023,Wang2022RationaleEnsembles,Chen2023UniversalSelfConsistency}.
More Agents Is All You Need studies repeated agent calls
\citep{Li2024MoreAgents}; Large Language Monkeys studies how repeated sampling
changes task coverage \citep{Brown2024LargeLanguageMonkeys}; and work on
compound-system scaling and best-of-\(N\) characterizes how performance changes
with inference-time compute \citep{Chen2024MoreCalls,Huang2025BestOfN}.
These papers motivate the coverage side of our analysis. Additional samples
matter only when they expose candidates that were absent from a single draw.

{\bf Verification and selection.\enspace}
A separate line of work asks how to choose among sampled candidates. Learned
verifiers select mathematical solutions \citep{Cobbe2021}; process supervision
scores intermediate reasoning steps
\citep{Lightman2024,Wang2024MathShepherd}; and pairwise rankers or verifier
ensembles combine imperfect selection signals
\citep{Jiang2023LLMBlender,Lifshitz2025MAV,SaadFalcon2025Weaver,Lee2026FUSE}.
LLM-as-a-judge and reward-model benchmarks show that selectors have their own
biases and failure modes
\citep{Zheng2023,Lambert2024RewardBench,Gu2024LLMJudgeSurvey}.
This is the distinction captured by identifiability in our theory: a useful
candidate can be present without being recoverable by the available selector.

{\bf Search, tools, and agentic systems.\enspace}
Tree of Thoughts, RAP, LATS, ReAct, Reflexion, and Self-Refine operate over
intermediate states using search, feedback, planning, or revision
\citep{Yao2023ToT,Hao2023RAP,Zhou2023LATS,Yao2023ReAct,
Shinn2023Reflexion,Madaan2023SelfRefine}. CRITIC, Chain-of-Verification,
Toolformer, and HuggingGPT add external evidence or computation
\citep{Gou2024CRITIC,Dhuliawala2024CoVe,Schick2023Toolformer,
Shen2023HuggingGPT}. Multi-agent frameworks such as CAMEL, AutoGen, MetaGPT,
Mixture-of-Agents, Archon, and Smoothie explore broader orchestration spaces
\citep{Li2023CAMEL,Wu2023AutoGen,Hong2024MetaGPT,
Wang2024MixtureAgents,SaadFalcon2024Archon,Guha2024Smoothie}.
Our analysis abstracts away from these architectural choices and isolates three
requirements for reliable amplification: useful moves must be proposed,
distinguishable, and composable along a bounded-progress trajectory.

{\bf Verifier-backed code and formal reasoning.\enspace}
Code and formal reasoning are natural settings for this question because
candidate artifacts can often be checked by execution, tests, types, or proof
assistants. AlphaCode uses large-scale sampling and behavioral filtering
\citep{Li2022AlphaCode}; SWE-bench and systems such as SWE-agent,
AutoCodeRover, and Agentless study repository-level repair with structured
generation and validation
\citep{Jimenez2024SWEbench,Yang2024SWEagent,
Zhang2024AutoCodeRover,Xia2025Agentless}. These settings make it possible to
measure separately whether a correct candidate was generated and whether it
was selected.

{\bf Benchmarking implication.\enspace}
For verifier-backed tasks, reporting pass@1 alone conflates proposal quality
with selection. We therefore report pass@1, oracle best-of-\(k\), implemented
system performance, oracle-gap recovery, and resource use. Together these
quantities separate one-shot success, proposal-pool coverage, selection quality,
and the cost of approaching the proposal-side ceiling.

\end{document}